\documentclass[final,12pt]{alt2025}

\usepackage{times}
\usepackage{mathtools}
\usepackage[mathscr]{euscript}
\usepackage{thmtools}
\usepackage{thm-restate}
\usepackage{booktabs}  
\usepackage{caption}
\usepackage{graphicx}
\usepackage{wrapfig}

\usepackage{bm}

\let\vec\undefined
\usepackage{MnSymbol}
\DeclareMathAlphabet\mathbb{U}{msb}{m}{n}
\usepackage{xpatch}

\def\Rset{\mathbb{R}}

\let\Pr\undefined

\DeclareMathOperator*{\Pr}{\mathbb{P}}

\DeclareMathOperator*{\E}{\mathbb E}
\DeclareMathOperator*{\argmax}{argmax}
\DeclareMathOperator*{\argmin}{argmin}

\DeclarePairedDelimiter{\abs}{\lvert}{\rvert} 
\DeclarePairedDelimiter{\bracket}{[}{]}
\DeclarePairedDelimiter{\curl}{\{}{\}}

\DeclarePairedDelimiter{\paren}{(}{)}

\newcommand{\sC}{{\mathscr C}}
\newcommand{\sD}{{\mathscr D}}
\newcommand{\sE}{{\mathscr E}}
\newcommand{\sF}{{\mathscr F}}

\newcommand{\sH}{{\mathscr H}}

\newcommand{\sM}{{\mathscr M}}

\newcommand{\sX}{{\mathscr X}}
\newcommand{\sY}{{\mathscr Y}}

\newcommand{\sfD}{{\mathsf D}}

\newcommand{\sfL}{{\mathsf L}}

\newcommand{\Rad}{\mathfrak R}

\newcommand{\h}{\widehat}
\newcommand{\ov}{\overline}

\newcommand{\e}{\epsilon}

\newcommand{\ignore}[1]{}

\DeclareMathOperator{\sign}{sign}

\newcommand{\hh}{{\sf h}}

\hypersetup{
  breaklinks   = true, 
  colorlinks   = true, 
  urlcolor     = blue, 
  linkcolor    = blue, 
  citecolor   = blue 
}

\usepackage[toc, page, header]{appendix}
\newcommand{\lbi}{{\sfL}}

\newcommand{\pred}{{\mathrm{pred}}}

\usepackage{times}

\title[Enhanced $\sH$-Consistency Bounds]{Enhanced $\sH$-Consistency Bounds}

\altauthor{%
 \Name{Anqi Mao} \Email{aqmao@cims.nyu.edu}\\
 \addr Courant Institute of Mathematical Sciences, New York%
 \AND
 \Name{Mehryar Mohri} \Email{mohri@google.com}\\
 \addr Google Research and Courant Institute of Mathematical Sciences, New York%
 \AND
 \Name{Yutao Zhong} \Email{yutao@cims.nyu.edu}\\
 \addr Google Research, New York%
}

\begin{document}

\maketitle

\begin{abstract}%
Recent research has introduced a key notion of $\sH$-consistency
bounds for surrogate losses. These bounds offer finite-sample
guarantees, quantifying the relationship between the zero-one
estimation error (or other target loss) and the surrogate loss
estimation error for a specific hypothesis set. However, previous
bounds were derived under the condition that a lower bound of the
surrogate loss conditional regret is given as a convex function of the
target conditional regret, without non-constant factors depending on
the predictor or input instance. Can we derive finer and more
favorable $\sH$-consistency bounds? In this work, we relax this
condition and present a general framework for establishing
\emph{enhanced $\sH$-consistency bounds} based on more general
inequalities relating conditional regrets. Our theorems not only
subsume existing results as special cases but also enable the
derivation of more favorable bounds in various scenarios. These
include standard multi-class classification, binary and multi-class
classification under Tsybakov noise conditions, and bipartite ranking.
\end{abstract}


\section{Introduction}
\label{sec:intro}

The design of accurate and reliable learning algorithms hinges on the
choice of surrogate loss functions, since optimizing the true target
loss is typically intractable.  A key property of these surrogate
losses is Bayes-consistency, which guarantees that minimizing the
surrogate loss leads to the minimization of the true target loss in
the limit. This property has been well-studied for convex margin-based
losses in both binary \citep{Zhang2003,bartlett2006convexity} and
multi-class classification settings
\citep{tewari2007consistency}. However, this classical notion has
significant limitations since it only holds asymptotically and for the
impractical set of all measurable functions. Thus, it fails to provide
guarantees for real-world scenarios where learning is restricted to
specific hypothesis sets, such as linear models or neural networks.
In fact, Bayes-consistency does not always translate into superior
performance, as highlighted by \citet{long2013consistency}.

Recent research has addressed these limitations by introducing
$\sH$-consistency bounds
\citep*{zhong2025fundamental,awasthi2022Hconsistency,awasthi2022multi,mao2023cross,
  MaoMohriZhong2023characterization,MaoMohriZhong2023structured}. These
bounds offer non-asymptotic guarantees, quantifying the relationship
between the zero-one estimation error (or other target loss) and the
surrogate loss estimation error for a specific hypothesis set. While
existing work has characterized the general behavior of these bounds
\citep{MaoMohriZhong2024}, particularly for smooth surrogates in binary
and multi-class classification, their derivation has been restricted
by certain assumptions. Specifically, previous bounds were derived
under the condition that a lower bound of the surrogate loss
conditional regret is given as a convex function of the target
conditional regret, without non-constant factors depending on the
predictor or input instance. Can we derive finer and more favorable
$\sH$-consistency bounds?

In this work, we relax this condition and present a general framework
for establishing \emph{enhanced $\sH$-consistency bounds} based on more
general inequalities relating conditional regrets. Our theorems not
only subsume existing results as special cases but also enable the
derivation of tighter bounds in various scenarios. These include
standard multi-class classification, binary and multi-class
classification under Tsybakov noise conditions, and bipartite ranking.

The remainder of this paper is organized as follows. In
Section~\ref{sec:general-tools}, we prove general theorems serving as
new fundamental tools for deriving enhanced $\sH$-consistency bounds.
These theorems allow for the presence of non-constant factors $\alpha$
and $\beta$ which can depend on both the hypothesis $h$ and the input
instance $x$. They include as special cases previous $\sH$-consistency
theorems, where $\alpha \equiv 1$ and $\beta \equiv 1$.  Furthermore,
the bounds of these theorems are tight.
In Section~\ref{sec:standard-multi}, we apply these tools to establish
enhanced $\sH$-consistency bounds for constrained losses in standard
multi-class classification. These bounds are enhanced by incorporating
a new hypothesis-dependent quantity, $\Lambda(h)$, not present in
previous work.
Next, in Section~\ref{sec:low-noise}, we derive a series of new and
substantially more favorable $\sH$-consistency bounds under Tsybakov
noise conditions. Our bounds in binary classification
(Section~\ref{sec:low-noise-binary}) recover as special cases some
past results and even improve upon some.  Our bounds for multi-class
classification (Section~\ref{sec:low-noise-multi}) are entirely new
and do not admit any past counterpart even in special cases. To
illustrate the applicability of our results, we instantiate them for
common surrogate losses in both binary and multi-class classification.
\ignore{(see Appendix~\ref{app:example})}

In Section~\ref{sec:bi-ranking}, we extend our new fundamental tools to
the bipartite ranking setting (Section~\ref{sec:bi-tools}) and
leverage them to derive novel $\sH$-consistency bounds relating
classification surrogate losses to bipartite ranking surrogate losses.
We also identify a necessary condition for loss functions to admit
such bounds.  We present a remarkable direct upper bound on the
estimation error of the RankBoost loss function, expressed in terms of
the AdaBoost loss, with a multiplicative factor equal to the
classification error of the predictor
(Section~\ref{sec:bi-exp}). Additionally, we prove another surprising
result with a different non-constant factor for logistic regression
and its ranking counterpart (Section~\ref{sec:bi-log}). Conversely, we
establish negative results for such bounds in the case of the hinge
loss (Section~\ref{sec:bi-hinge}).

In Appendix~\ref{app:generalization}, we provide novel enhanced
generalization bounds. We provide a detailed discussion of related
work in Appendix~\ref{app:related-work}. We begin by establishing the
necessary terminology and definitions.

\section{Preliminaries}
\label{sec:pre}

We consider the standard supervised learning setting. Consider $\sX$
as the input space, $\sY$ as the label space, and $\sD$ as a
distribution over $\sX \times \sY$. Given a sample $S = \paren*{(x_1,
  y_1), \ldots, (x_m, y_m)}$ draw i.i.d.\ according to $\sD$, our goal
is to learn a hypothesis $h$ that maps $\sX$ to a prediction space,
denoted by $\pred$. This hypothesis is chosen from a predefined
hypothesis set $\sH$, which is a subset of the family of all
measurable functions, denoted by $\sH_{\rm{all}} = \curl*{h \colon \sX
  \to \pred \mid h \text{ measurable}}$.  We denote by $\ell \colon
\sH \times \sX \times \sY \to \Rset_{+}$ the loss function that
measures the performance of a hypothesis $h$ on any pair $(x,
y)$. Given a loss function $\ell$ and a hypothesis set $\sH$, we
denote by $\sE_{\ell}(h) = \E_{(x, y) \sim \sD} \bracket*{\ell(h, x,
  y)}$ the generalization error and by $\sE^*_{\ell}(\sH) = \inf_{h
  \in \sH} \sE_{\ell}(h)$ the best-in-class generalization error. We
further define the conditional error and the best-in-class condition
error as $\sC_{\ell}(h, x) = \E_{y \mid x} \bracket*{\ell(h, x, y)}$
and $\sC^*_{\ell}(\sH, x) = \inf_{h \in \sH} \sC_{\ell}(h, x)$,
respectively. Thus, the generalization error can be rewritten as
$\sE_{\ell}(h) = \E_X \bracket*{\sC_{\ell}(h, x)}$. For convenience,
we refer to $\sE_{\ell}(h) - \sE^*_{\ell}(\sH)$ as the
\emph{estimation error}, to $\sE^*_{\ell}(\sH) - \sE^*_{\ell}\paren*{\sH_{\rm{all}}}$ as the
\emph{estimation error} and to $\Delta \sC_{\ell, \sH}(h, x) \coloneqq
\sC_{\ell}(h, x) - \sC^*_{\ell}(\sH, x)$ as the \emph{conditional
regret}.

Minimizing the target loss function, as specified by the learning
task, is typically NP-hard. Instead, a surrogate loss function is
often minimized. This paper investigates how minimizing surrogate
losses can guarantee the minimization of the target loss function.  We
are especially interested in three applications: binary
classification, multi-class classification, and bipartite ranking,
although our general results are applicable to any supervised learning
framework.

\textbf{Binary classification.} Here, the label space is $\sY =
\curl*{\minus 1, \plus 1}$, and the prediction space is $\pred =
\Rset$. The target loss function is the binary zero-one loss, defined
by $\ell^{\rm{bi}}_{0-1}(h, x, y) = 1_{\sign(h(x)) \neq y}$, where
$\sign(t) = 1$ if $t \geq 0$ and $-1$ otherwise. Let $\eta(x) =
\mathbb{P} (Y = \plus 1 \mid X = x)$ be the conditional probability of
$Y = \plus 1$ given $X = x$. The condition error can be expressed
explicitly as $\sC_{\ell}(h, x) = \eta(x) \ell(h, x, +1) + \paren*{1 -
  \eta(x)} \ell(h, x, -1)$. Common surrogate loss functions include
the margin-based loss functions $\ell_{\Phi}(h, x, y) = \Phi(yh(x))$,
for some function $\Phi$ that is non-negative and non-increasing.

\textbf{Multi-class classification.} Here, the label space is $[n]
\coloneqq \curl*{1, \ldots, n}$, and the prediction space is $\pred =
\Rset^n$ for some $n \in \mathbb{Z}_{+}$. Let $h(x, y)$ denote the
$y$-th element of $h(x)$, where $y \in [n]$. The target loss function
is the multi-class zero-one loss, defined by $\ell_{0-1}(h, x, y) =
1_{\hh(x) \neq y}$, where $\hh(x) = \argmax_{y \in \sY} h(x, y)$. An
arbitrary but fixed deterministic strategy is used for breaking
ties. For simplicity, we fix this strategy to select the label with
the highest index under the natural ordering of labels.  Let $p(y \mid
x) = \mathbb{P} (Y = y \mid X = x)$ be the conditional probability of
$Y = y$ given $X = x$. The condition error can be explicitly expressed
as $\sC_{\ell}(h, x) = \sum_{y \in \sY} p(y \mid x) \ell(h, x,
y)$. Common surrogate loss functions include the max losses
\citep{crammer2001algorithmic}, constrained losses
\citep{lee2004multicategory}, and comp-sum losses
\citep{mao2023cross}.

\textbf{Bipartite ranking.} Here, the label space is $\sY =
\curl*{\minus 1, \plus 1}$, and the prediction space is $\pred =
\Rset$.  Unlike the previous two settings, the goal here is to
minimize the bipartite misranking loss $\sfL_{0-1}$, defined for any
two pairs $(x, y)$ and $(x', y')$ drawn i.i.d.\ according to $\sD$,
and a hypothesis $h$: $\sfL_{0-1}(h, x, x', y, y') = 1_{(y - y')(h(x)
  - h(x')) < 0} + \frac{1}{2} 1_{(h(x) = h(x')) \wedge (y\neq y')}$.
Let $\eta(x) = \mathbb{P} (Y = \plus 1 \mid X = x)$ be the conditional
probability of $Y = \plus 1$ given $X = x$. Given a loss function
$\sfL \colon \sH \times \sX \times \sX \times \sY \times \sY \to
\Rset_{+}$ and a hypothesis set $\sH$, the generalization error and
the condition error can be defined accordingly as $\sE_{\sfL}(h) =
\E_{(x, y) \sim \sD, (x', y') \sim \sD}\bracket*{\sfL(h, x, x', y,
  y')}, \ov \sC_{\sfL}(h, x, x') = \eta(x)(1 - \eta(x')) \sfL(h, x,
x', \plus 1, \minus 1) + \eta(x')(1 - \eta(x)) \sfL(h, x, x', \minus
1, \plus 1)$.  The best-in-class generalization error and
best-in-class condition error can be expressed as $\sE^*_{\sfL}(\sH) =
\inf_{h \in \sH} \sE_{\sfL}(h)$ and $\ov \sC^*_{\sfL}(\sH, x, x') =
\inf_{h \in \sH} \sC_{\sfL}(h, x, x')$, respectively. The estimation
error and conditional regret can be written as $\sE_{\sfL}(h) -
\sE^*_{\sfL}(\sH)$ and $\Delta \ov \sC_{\sfL, \sH}(h, x, x') = \ov
\sC_{\sfL}(h, x, x') - \ov \sC^*_{\sfL}(\sH, x, x')$, respectively.
Common bipartite ranking surrogate loss functions typically take the
following form: $\sfL_{\Phi}(h, x, x', y, y') = \Phi \paren*{\frac{(y
    - y') \paren*{h(x) - h(x')}}{2}} 1_{y \neq y'}$, for some function
$\Phi$ that is non-negative and non-increasing. Another choice is to
use the margin-based loss $\ell_{\Phi}(h, x, y) = \Phi(y h(x))$ in
binary classification as a surrogate loss. We will specifically be
interested in the guarantees of minimizing $\ell_{\Phi}$ with respect
to the minimization of $\sfL_{\Phi}$.

\section{New fundamental tools for $\sH$-consistency bounds}
\label{sec:general-tools}

This section introduces new tools for deriving finer and more general
$\sH$-consistency bounds. We begin with a brief overview of
$\sH$-consistency.

\textbf{Background on $\sH$-Consistency bounds.} A desirable property of surrogate loss functions is \emph{Bayes-consistency}
\citep{Zhang2003,bartlett2006convexity,steinwart2007compare,tewari2007consistency}. Bayes-consistency ensures that, asymptotically, minimizing a surrogate
loss $\ell_1$ over all measurable functions, denoted by
$\sH_{\rm{all}}$, leads to the minimization of the target loss
function $\ell_2$ over the same function family:
\begin{equation*}
  \sE_{\ell_1}(h_n)
  - \sE^*_{\ell_1}(\sH_{\rm{all}}) \xrightarrow{n \rightarrow +\infty} 0
  \implies
  \sE_{\ell_2}(h_n) -
  \sE^*_{\ell_2}(\sH_{\rm{all}}) \xrightarrow{n \rightarrow +\infty} 0.
\end{equation*}
However, Bayes-consistency is an asymptotic property, providing no
guarantees for approximate minimizers. Additionally, it applies only
to the family of all measurable functions, which is less relevant in
practical scenarios where restricted hypothesis sets $\sH$ are used.
To address these limitations,
\citet{awasthi2022Hconsistency,awasthi2022multi} proposed a more
refined framework, called \emph{$\sH$-consistency bounds}. These
bounds provide upper bounds on the target estimation error in terms of
the surrogate estimation error for a concave function $\Gamma \geq 0$ with $\Gamma(0) = 0$:
\begin{equation}
  \label{eq:H-consistency-bound}
  \sE_{\ell_2}(h) - \sE^*_{\ell_2}(\sH) +
  \sM_{\ell_2}(\sH)\leq \Gamma\paren*{\sE_{\ell_1}(h)
    - \sE^*_{\ell_1}(\sH) + \sM_{\ell_1}(\sH)},
\end{equation}
where $\sM_{\ell}(\sH) = \sE^*_{\ell}(\sH) - \E_X
\bracket*{\sC^*_{\ell}(\sH, x)} \geq 0$ represents the
\emph{minimizability gap}, which measures the difference between the
best-in-class generalization error and the expected best-in-class
conditional error.  This concept can also be adapted to the bipartite
ranking setting, with $\sM_{\sfL}(\sH) = \sE^*_{\sfL}(\sH) - \E_{(x,
  x')} \bracket*{\ov \sC^*_{\sfL}(\sH, x, x')}$.

The minimizability gap is always upper bounded by the approximation
error but it is generally a more fine-grained measure
\citep{MaoMohriZhong2023characterization}.  When $\sH =
\sH_{\rm{all}}$ or $\sE^*_{\ell}(\sH) = \sE^*_{\ell}(\sH_{\rm{all}})$,
the minimizability gaps vanish \citep{steinwart2007compare} leading to
\emph{excess error bounds} that imply Bayes-consistency, by taking the
limit. However, in general, minimizability gaps are non-zero and
represent an inherent quantity depending on the distribution and the
hypothesis.

Thus, $\sH$-consistency bounds provide a stronger and more informative
guarantee than Bayes-consistency, since they are both non-asymptotic
and specific to the hypothesis set $\sH$ used. Note that, by the
sub-additivity of a concave function $\Gamma \geq 0$, an
$\sH$-consistency bound also implies
\[
\sE_{\ell_2}(h) - \sE^*_{\ell_2}(\sH) \leq
\Gamma\paren*{\sE_{\ell_1}(h) - \sE^*_{\ell_1}(\sH)} + \Gamma
\paren*{\sM_{\ell_1}(\sH)} - \sM_{\ell_2}(\sH),
\]
where $\Gamma \paren*{\sM_{\ell_1}(\sH)} - \sM_{\ell_2}(\sH)$ is an
inherent constant depending on the hypothesis set and distribution.
The ultimate algorithmic goal when using a surrogate loss $\ell_1$ is
to minimize the estimation loss $\bracket*{\sE_{\ell_1}(h) -
  \sE^*_{\ell_1}(\sH)}$. An $\sH$-consistency bound ensures that
reducing this error to $\e$ implies that the target estimation loss
$\bracket*{\sE_{\ell_2}(h) - \sE^*_{\ell_2}(\sH)}$ is upper bounded by
$\Gamma(\e) + \Gamma \paren*{\sM_{\ell_1}(\sH)} - \sM_{\ell_2}(\sH)$,
or just $\Gamma(\e)$ when the minimizability gaps vanish. Recent work
by \cite{MaoMohriZhong2024} shows that for all smooth surrogate losses
in binary classification, $\Gamma(\epsilon)$ behaves as
$\sqrt{\epsilon}$ near zero.

\textbf{Enhanced $\sH$-consistency bounds and tools.}  While
$\sH$-consistency bounds offer strong, non-asymptotic guarantees
tailored to $\sH$, they can be further enhanced by considering a more
general form such as the following:
\begin{equation}
\label{eq:general-form}
  \sE_{\ell_2}(h) - \sE_{\ell_2}^*(\sH) + \sM_{\ell_2}(\sH)
  \leq \Gamma \paren*{\gamma(h)
  \paren*{\sE_{\ell_1}(h) - \sE^*_{\ell_1}(\sH) + \sM_{\ell_1}(\sH)} },
\end{equation}
where $\gamma(h)$ is a factor depending on the hypothesis $h$. This
refinement allows the bound to incorporate $h$-dependent information,
enabling the use of more favorable functions $\Gamma$, which can
improve the bound's behavior near zero. In the following sections, we
will demonstrate this for both classification and bipartite
ranking. For instance, we will show that under certain noise
conditions in classification, the behavior of $\Gamma$ can outperform
the typical square-root dependence, approaching near-linear behavior.

The foundation of earlier $\sH$-consistency bounds involves finding a
convex function $\Psi$ or a concave function $\Gamma$ such that:
$\Psi\paren*{\Delta\sC_{\ell_2,\sH}(h, x)} \leq \Delta
\sC_{\ell_1,\sH}(h, x)$ or $\Delta\sC_{\ell_2,\sH}(h, x) \leq \Gamma
\paren*{\Delta \sC_{\ell_1,\sH}(h, x)}$. We extend this approach by
relaxing the inequalities to incorporate functions $\alpha(h, x)$ and
$\beta(h, x)$ that depend on both the hypothesis and the input
instance. The following two theorems illustrate this enhancement with
general guarantees of the form \eqref{eq:general-form} derived from
such relaxed inequalities, where $\gamma(h)$ is defined in terms of
$\alpha$ and $\beta$.

\begin{restatable}{theorem}{NewBoundConvex}
\label{Thm:new-bound-convex}
Assume that there exist a convex function $\Psi\colon \Rset_+ \to
\Rset$ and two positive functions $\alpha\colon \sH
\times \sX \to \Rset^*_+$ and $\beta\colon \sH
\times \sX \to \Rset^*_+$ with $\sup_{x \in \sX} \alpha(h, x) <
+\infty$ and $\E_{x \in \sX} \bracket*{\beta(h, x)} = 1$ for all $h \in \sH$ such that the following holds for all
$h\in \sH$ and $x\in \sX$: $\Psi\paren*{\frac{\Delta\sC_{\ell_2,\sH}(h,
    x)}{\beta(h, x)}} \leq \alpha(h, x) \,
\Delta \sC_{\ell_1,\sH}(h, x)$. Then, the following inequality holds
for any hypothesis $h \in \sH$:
\begin{align}
\label{eq:new-bound-convex}
  \Psi\paren*{\sE_{\ell_2}(h) - \sE_{\ell_2}^*(\sH) + \sM_{\ell_2}(\sH)}
  \leq \gamma(h)
  \paren*{\sE_{\ell_1}(h) - \sE^*_{\ell_1}(\sH) + \sM_{\ell_1}(\sH)}.
\end{align}
with $\gamma(h) = \bracket*{\sup_{x \in \sX} \alpha(h, x)
    \beta(h, x)}$. If, additionally,
$\sX$ is a subset of $\Rset^n$ and, for any $h \in \sH$, $x \mapsto
\Delta\sC_{\ell_1,\sH}(h, x)$ is non-decreasing and $x \mapsto
\alpha(h, x) \beta(h, x)$ is non-increasing, or vice-versa, then, the
inequality holds with $\gamma(h) = \E_{X} \bracket*{\alpha(h, x) \beta(h, x)}$.
\end{restatable}

\begin{restatable}{theorem}{NewBoundConcave}
\label{Thm:new-bound-concave}
Assume that there exist a concave function $\Gamma \colon \Rset_+ \to
\Rset$ and two positive functions $\alpha\colon \sH
\times \sX \to \Rset^*_+$ and $\beta\colon \sH
\times \sX \to \Rset^*_+$ with $\sup_{x \in \sX} \alpha(h, x) <
+\infty$ and $\E_{x \in \sX} \bracket*{\beta(h, x)} = 1$ for all $h \in \sH$ such that the following holds for all
$h\in \sH$ and $x\in \sX$: $\frac{\Delta\sC_{\ell_2,\sH}(h,
  x)}{\beta(h, x)}
\leq \Gamma \paren*{\alpha(h, x) \,
\Delta \sC_{\ell_1,\sH}(h, x)}$. Then, the following inequality holds
for any hypothesis $h \in \sH$:
\begin{align}
\label{eq:new-bound-concave}
  \sE_{\ell_2}(h) - \sE_{\ell_2}^*(\sH) + \sM_{\ell_2}(\sH)
  \leq \Gamma \paren*{\gamma(h)
  \paren*{\sE_{\ell_1}(h) - \sE^*_{\ell_1}(\sH) + \sM_{\ell_1}(\sH)} },
\end{align}
with $\gamma(h) = \bracket*{\sup_{x \in \sX} \alpha(h, x)
    \beta(h, x)}$.  If, additionally,
$\sX$ is a subset of $\Rset^n$ and, for any $h \in \sH$, $x \mapsto
\Delta\sC_{\ell_1,\sH}(h, x)$ is non-decreasing and $x \mapsto
\alpha(h, x) \beta(h, x)$ is non-increasing, or vice-versa, then, the
inequality holds with $\gamma(h) = \E_{X} \bracket*{\alpha(h, x)
    \beta(h, x)}$.
\end{restatable}
We refer to Theorems~\ref{Thm:new-bound-convex} and
\ref{Thm:new-bound-concave} as \emph{new fundamental tools} because
they incorporate additional factors, $\alpha$ and $\beta$, which
depend on both the hypothesis $h$ and the instance $x$.  These
theorems generalize previous results from
\citep{awasthi2022Hconsistency,awasthi2022multi}, which can be
recovered as special cases when $\alpha \equiv 1$ and $\beta \equiv
1$. Compared to earlier approaches, these new tools offer more precise
$\sH$-consistency bounds in familiar settings and extend them to new
scenarios where previous methods are insufficient. We will demonstrate
their applications in both contexts. Moreover, the bounds derived
using these tools are \emph{tight}.

\begin{restatable}{lemma}{NewBoundTight}
\label{Thm:new-bound-tight}
The bounds of Theorems~\ref{Thm:new-bound-convex} and
\ref{Thm:new-bound-concave} are \emph{tight} in the following sense:
for some distributions, Inequality \eqref{eq:new-bound-convex}
(respectively Inequality \eqref{eq:new-bound-concave}) is the tightest
possible $\sH$-consistency bound that can be derived under the
assumption of Theorem~\ref{Thm:new-bound-convex} (respectively Theorem
\ref{Thm:new-bound-concave}).
\end{restatable}
Note that when $\Gamma(0) \geq 0$, the concave function $\Gamma$ is sub-additive over $\Rset_{+}$, and the theorem implies the following inequality:
\begin{align*}
  \sE_{\ell_2}(h) - \sE^*_{\ell_2}(\sH) + \sM_{\ell_2}(\sH)
  & \leq \Gamma \paren*{\gamma(h) \paren*{\sE_{\ell_1}(h) - \sE^*_{\ell_1}(\sH)} }
+ \Gamma \paren*{ \gamma(h) \, \sM_{\ell_1}(\sH) },
\end{align*}
\ignore{with $\gamma(h) = \E_{X} \bracket*{\frac{\alpha(h, x) \beta(h,
    x)}{\E_X\bracket*{\beta(h, x)}}}$.} The bound implies that if the
surrogate estimation loss of a predictor $h$ is reduced to $\e$, then
the target estimation loss is bounded by $\Gamma(\gamma(h) \, \e) +
\Gamma \paren*{ \gamma(h) \, \sM_{\ell_1}(\sH) } - \sM_{\ell_2}(\sH)$.
When the minimizability gaps are zero, for example when the
problem is realizable, the upper bound simplifies to $\Gamma(\gamma(h) \, \e)$.
In the special case of $\Psi(x) = x^s$ or equivalently, $\Gamma(x) =
x^{\frac{1}{s}}$, for some $s \geq 1$ with conjugate number $t \geq 1$,
that is $\frac{1}{s} + \frac{1}{t} = 1$, we can further obtain the following
result.
\begin{restatable}{theorem}{NewBoundPower}
\label{Thm:new-bound-power}
Assume that there exist two positive functions $\alpha\colon \sH
\times \sX \to \Rset^*_+$ and $\beta\colon \sH \times \sX \to
\Rset^*_+$ with $\sup_{x \in \sX} \alpha(h, x) < +\infty$ and $\E_{x
  \in \sX} \bracket*{\beta(h, x)} = 1$ for all $h \in \sH$ such
that the following holds for all $h\in \sH$ and $x\in \sX$: $
\frac{\Delta\sC_{\ell_2,\sH}(h, x)}{\beta(h, x)} \leq \paren*{\alpha(h, x) \, \Delta
  \sC_{\ell_1,\sH}(h, x)}^{\frac{1}{s}}$, for some $s \geq 1$ with
conjugate number $t \geq 1$, that is $\frac{1}{s} + \frac{1}{t} = 1$. Then,
for $\gamma(h) = \E_X \bracket*{\alpha^{\frac{t}{s}}(h, x)
    \beta^t(h, x)}^{\frac{1}{t}}$, the
following inequality holds for any $h \in \sH$:
\begin{align*}
  \sE_{\ell_2}(h) - \sE_{\ell_2}^*(\sH) + \sM_{\ell_2}(\sH)
  \leq \gamma(h) \bracket*{\sE_{\ell_1}(h) - \sE^*_{\ell_1}(\sH)
    + \sM_{\ell_1}(\sH)}^{\frac{1}{s}}.
\end{align*}
\end{restatable}
As above, by the sub-additivity of
$x \mapsto x^{\frac{1}{s}}$ over
$\Rset_+$, the bound implies
\begin{align*}
  \sE_{\ell_2}(h) - \sE_{\ell_2}^*(\sH) + \sM_{\ell_2}(\sH)
  \leq \gamma(h) 
  \bracket*{\paren*{\sE_{\ell_1}(h) - \sE^*_{\ell_1}(\sH)}^{\frac{1}{s}}
  + \paren*{ \sM_{\ell_1}(\sH)}^{\frac{1}{s}}}.
\end{align*}

The proofs of Theorems~\ref{Thm:new-bound-convex},
\ref{Thm:new-bound-concave},
\ref{Thm:new-bound-power}, and Lemma~\ref{Thm:new-bound-tight} are presented in
Appendix~\ref{app:general}. These proofs are more complex than their
counterparts for earlier results in the literature due to the presence
of the functions $\alpha$ and $\beta$. Our proof technique involves a
refined application of Jensen's inequality tailored to the $\beta$
function, the use of H\"older's inequality adapted for the $\alpha$
function, and the application of the FKG Inequality in the second part
of both Theorems~\ref{Thm:new-bound-convex} and
\ref{Thm:new-bound-concave}. The proof of
Theorem~\ref{Thm:new-bound-power} also leverages H\"older's
Inequality. For cases where $\Psi$ or $\Gamma$ is linear, our proof
shows that the resulting bounds are essentially optimal, modulo the
use of H\"older's inequality. As we shall see in
Section~\ref{sec:low-noise}, $\Psi$ and $\Gamma$ are linear
when Massart's noise assumption holds.

Building upon these theorems, we proceed to derive
finer $\sH$-consistency bounds than existing ones.
\ignore{The significance of our work extends beyond the introduction of new fundamental tools in Section~\ref{sec:general-tools}. It also lies in the innovative application of these tools to diverse scenarios, resulting in novel and impactful theoretical guarantees across various domains: standard multi-class classification (Section~\ref{sec:standard-multi}), binary and multi-class classification under Tsybakov noise conditions (Section~\ref{sec:low-noise}), and bipartite ranking (Section ~\ref{sec:bi-ranking}).}

\section{Standard multi-class classification}
\label{sec:standard-multi}

We first apply our new tools to establish enhanced
$\sH$-consistency bounds in standard multi-class classification. We
will consider the constrained losses \citep{lee2004multicategory},
defined as
\begin{equation}
\label{eq:lee_loss}
\Phi^{\rm{cstnd}}(h, x, y) = \sum_{y'\neq y} \Phi \paren*{-h(x, y')} \text{ subject to } \sum_{y\in \sY} h(x, y) = 0,
\end{equation}
where $\Phi$ is a non-increasing and non-negative function. We will
specifically consider $\Phi(u) = e^{-u}$, $\Phi(u) = \max \curl*{0, 1
  - u}$, and $\Phi(u) = (1 - u)^2 1_{u \leq 1}$, corresponding to the
constrained exponential loss, constrained hinge loss, and constrained
squared hinge loss, respectively. By applying
Theorems~\ref{Thm:new-bound-convex} or \ref{Thm:new-bound-concave}, we
obtain the following enhanced $\sH$-consistency bounds.  We say that a
hypothesis set is symmetric if there exists a family $\sF$ of
functions $f$ mapping from $\sX$ to $\Rset$ such that
$\curl*{\bracket*{h(x, 1), \ldots, h(x, n + 1)} \colon h \in \sH} =
\curl*{\bracket*{f_1(x),\ldots, f_{n + 1}(x)} \colon f_1, \ldots, f_{n
    + 1} \in \sF}$, for any $x \in \sX$. We say that a hypothesis set
$\sH$ is complete if for any $(x, y) \in \sX \times \sY$, the set of
scores generated by it spans across the real numbers: $\curl*{h(x, y)
  \mid h \in \sH} = \Rset$.
  
\begin{restatable}[\textbf{Enhanced $\sH$-consistency bounds for constrained losses}]
  {theorem}{BoundLeeFiner}
\label{Thm:bound_lee_finer}
Assume that $\sH$ is symmetric and complete. Then, the following
inequality holds for any hypothesis $h \in \sH$:
\begin{align*}
 \sE_{\ell_{0-1}}(h) - \sE^*_{\ell_{0-1}}(\sH) + \sM_{\ell_{0-1}}(\sH) \leq \Gamma \paren*{\sE_{\Phi^{\mathrm{cstnd}}}(h) - \sE^*_{\Phi^{\mathrm{cstnd}}}(\sH) + \sM_{\Phi^{\mathrm{cstnd}}}(\sH)},
\end{align*}
where $\Gamma(x) = \frac{ \sqrt{2}\, x^{\frac{1}{2}}}{\paren*{{e^{\Lambda(h)}}}^{\frac{1}{2}}}$ for $\Phi(u) = e^{-u}$, $\Gamma(x) = \frac{x}{1 + \Lambda(h)}$ for $\Phi(u) = \max \curl*{0, 1 - u}$, and $\Gamma(x) = \frac{ x^{\frac{1}{2}}}{1 + \Lambda(h)}$ for $\Phi(u) = (1 - u)^2 1_{u \leq 1}$. Additionally, $\Lambda(h) = \inf_{x \in \sX} \max_{y\in \sY} h(x,y)$.
\end{restatable}
The proof is included in Appendix~\ref{app:bound_lee_finer}. These
$\sH$-consistency bounds are referred to as enhanced $\sH$-consistency
bounds because they incorporate a hypothesis-dependent quantity,
$\Lambda(h)$, unlike the previous $\sH$-consistency bounds derived for
the constrained losses in \citep{awasthi2022multi}. Since $\sum_{y \in
  \sY} h(x, y) = 0$, there must be non-negative scores. Consequently,
$\Lambda(h)$ must be greater than or equal to 0. Given that $\Gamma$
is non-decreasing, the $\sH$-consistency bounds in
Theorem~\ref{Thm:bound_lee_finer} are finer than the previous ones,
where $\Lambda(h)$ is replaced by zero.

\section{Classification under low-noise conditions}
\label{sec:low-noise}

The previous section demonstrated the usefulness of our new fundamental tools in deriving enhanced $\sH$-consistency bounds within standard
classification settings. In this section, we leverage them to
establish novel $\sH$-consistency bounds under low-noise conditions
for both binary and multi-class classification problems.

\subsection{Binary classification}
\label{sec:low-noise-binary}

Here, we first consider the binary classification setting under the
Tsybakov noise condition \citep{MammenTsybakov1999}, that is there
exist $B > 0$ and $\alpha \in [0, 1)$ such that
\[
\forall t > 0, \quad \Pr[\abs*{\eta(x) - 1/2} \leq t]
\leq B t^{\frac{\alpha}{1 - \alpha}}.
\]
Note that as $\alpha \to 1$, $t^{\frac{\alpha}{1 - \alpha}} \to 0$,
corresponding to Massart’s noise condition. When $\alpha = 0$, the
condition is void. This condition is equivalent to assuming the
existence of a universal constant $c > 0$ and $\alpha \in [0, 1)$ such
  that for all $h \in \sH$, the following inequality holds
  \citep{bartlett2006convexity}:
\begin{equation*}
\E[1_{\hh(X) \neq \hh^*(X)}]
\leq c \bracket*{\sE_{\ell^{\rm{bi}}_{0-1}}(h) - \sE_{\ell^{\rm{bi}}_{0-1}}(h^*)}^\alpha.  
\end{equation*}
where $h^*$ is the Bayes-classifier.
We also assume that there is no approximation error
and that $\sM_{\ell^{\rm{bi}}_{0-1}}(\sH) = 0$.
We refer to this as the \emph{Tsybakov noise assumption} in binary classification.

\begin{restatable}{theorem}{TsybakovBinary}
  \label{Thm:tsybakov-binary}
  Consider a binary classification setting where the Tsybakov noise
  assumption holds. Assume that the following holds for all $h \in
  \sH$ and $x \in \sX$: $\Delta\sC_{\ell^{\rm{bi}}_{0-1},\sH}(h, x)
  \leq \Gamma \paren*{\Delta \sC_{\ell,\sH}(h, x)}$, with $\Gamma(x) =
  x^{\frac{1}{s}}$, for some $s \geq 1$.  Then, for any $h \in \sH$,
  \[
  \sE_{\ell^{\rm{bi}}_{0-1}}(h) - \sE^*_{\ell^{\rm{bi}}_{0-1}}(\sH)
  \leq c^{\frac{s - 1}{s - \alpha(s - 1)}}\bracket*{\sE_{\ell}(h) - \sE^*_{\ell}(\sH) + \sM_{\ell}(\sH)}^{\frac{1}{s - \alpha(s - 1)}}.
  \]
\end{restatable}
The theorem offers a substantially more favorable $\sH$-consistency
guarantee for binary classification.  While standard $\sH$-consistency
bounds for smooth loss functions rely on a square-root dependency ($s
= 2$), this work establishes a linear dependence when Massart's noise
condition holds ($\alpha \to 1$), and an intermediate rate between
linear and square-root for other values of $\alpha$ within the range
$(0, 1)$.

Our result is general and admits as special cases previous related
bounds. In particular, setting $s = 2$ and $\alpha = 1$ recovers the
$\sH$-consistency bounds of \citep{awasthi2022Hconsistency} under
Massart's noise.  Additionally, with $\sH = \sH_{\rm{all}}$, it
recovers the excess bounds under the Tsybakov noise condition of
\citep{bartlett2006convexity}, but with a more favorable factor of one
instead of $2^{\frac{s}{s - \alpha (s - 1)}}$, which is always greater
than one. Table~\ref{tab:example-binary} illustrates several specific
instances of our bounds for margin-based losses.

The proof is given in Appendix~\ref{app:tsybakov-binary}. It consists
of defining $\beta(h, x) = \frac{1_{h(x) \neq h^*(x)} + \e}{\E_X\bracket*{1_{h(x) \neq h^*(x)} + \e}}$ for a fixed $\e
> 0$ and proving the inequality
$\frac{\Delta\sC_{\ell^{\rm{bi}}_{0-1},\sH}(h, x)}{\beta(h, x)} \leq \paren*{\alpha(h, x)
  \, \Delta \sC_{\ell,\sH}(h, x)}^{\frac{1}{s}}$, where $\alpha(h, x)
= \E_X\bracket*{1_{h(x) \neq h^*(x)} + \e}^s$. The result then follows the application of our
new tools Theorem~\ref{Thm:new-bound-power}. Note that our
proof is novel and that previous general tools for deriving
$\sH$-consistency bounds in \citep{awasthi2022Hconsistency,
  awasthi2022multi, zheng2023revisiting} cannot be applied here since
$\alpha$ and $\beta$ are not constants.

\begin{table*}[t]
\vskip -.2in
  \centering
  \resizebox{\textwidth}{!}{
  \begin{tabular}{@{\hspace{0cm}}llll@{\hspace{0cm}}}
    \toprule
   Loss functions & $\Phi$  & $\Gamma$  & $\sH$-consistency bounds\\
    \midrule
    Hinge & $\Phi_{\mathrm{hinge}}(u) = \max\curl*{0, 1 - u}$ & $x^1$ & $ \sE_{\ell}(h) - \sE^*_{\ell}(\sH) + \sM_{\ell}(\sH)$\\
    Logistic & $\Phi_{\mathrm{log}}(u) = \log(1 + e^{-u})$   & $x^2$ & $ c^{\frac{1}{2 - \alpha}}\bracket*{\sE_{\ell}(h) - \sE^*_{\ell}(\sH) + \sM_{\ell}(\sH)}^{\frac{1}{2 - \alpha}}$   \\
    Exponential & $\Phi_{\mathrm{exp}}(u) = e^{-u}$    & $x^2$  & $ c^{\frac{1}{2 - \alpha}}\bracket*{\sE_{\ell}(h) - \sE^*_{\ell}(\sH) + \sM_{\ell}(\sH)}^{\frac{1}{2 - \alpha}}$  \\
    Squared-hinge  & $\Phi_{\mathrm{sq-hinge}}(u) = (1 - u)^2 1_{u \leq 1}$ & $x^2$ &    $ c^{\frac{1}{2 - \alpha}}\bracket*{\sE_{\ell}(h) - \sE^*_{\ell}(\sH) + \sM_{\ell}(\sH)}^{\frac{1}{2 - \alpha}}$    \\
    Sigmoid & $\Phi_{\mathrm{sig}}(u) = 1 - \tanh(ku), ~k > 0$ & $x^1$ & $ \sE_{\ell}(h) - \sE^*_{\ell}(\sH) + \sM_{\ell}(\sH)$\\
    $\rho$-Margin & $\Phi_{\rho}(u) = \min \curl*{1, \max\curl*{0, 1 - \frac{u}{\rho}}}, ~\rho>0$ & $x^1$ &  $ \sE_{\ell}(h) - \sE^*_{\ell}(\sH) + \sM_{\ell}(\sH)$\\
    \bottomrule
  \end{tabular}
  }
  \caption{Examples of enhanced $\sH$-consistency upper bounds under
    the Tsybakov noise assumption and with complete hypothesis sets,
    for margin-based loss functions $\ell(h, x, y) = \Phi(yh(x))$.}
\label{tab:example-binary}
\vskip -.2in
\end{table*}

\subsection{Multi-class classification}
\label{sec:low-noise-multi}

The original definition of the Tsybakov noise
\citep{MammenTsybakov1999} was given and analyzed in the binary
classification setting. Here, we give a natural extension of this
definition and analyze its properties in the general multi-class
classification setting. We denote by $y_{\max}  = \argmax_{y \in \sY} p(y \mid x)$.
Define the minimal margin for a point $x \in \sX$ as follows:
$\gamma(x) = \Pr(y_{\max} | x) - \sup_{y \neq y_{\max}} \Pr(y | x)$.
The Tsybakov noise model assumes that the probability of a small
margin occurring is relatively low, that is there exist $B > 0$ and $\alpha \in [0, 1)$ such that
\begin{equation}
\forall t > 0, \quad \Pr[\gamma(X) \leq t] \leq B t^{\frac{\alpha}{1 - \alpha}}.
\end{equation}
In the binary classification setting, where $\gamma(x) = 2 \eta(x) -
1$, this recovers the condition described in
Section~\ref{sec:low-noise-binary}. For $\alpha \to 1$,
$t^{\frac{\alpha}{1 - \alpha}} \to 0$, this corresponds to Massart's
noise condition in multi-class classification. When $\alpha = 0$, the
condition becomes void. Similar to the binary classification setting,
we can establish an equivalence assumption for the Tsybakov noise
model as follows. We denote the Bayes classifier by $h^*$.

\begin{restatable}{lemma}{Tsybakov}
\label{lemma:Tsybakov}
  The Tsybakov noise assumption implies that there exists a constant $c$
  such that the following inequalities hold for any $h \in \sH$:
  \[
  \E[1_{\hh(x) \neq \hh^*(x)}]
  \leq c \E[\gamma(X) 1_{\hh(x) \neq \hh^*(x)}]^\alpha
  \leq c [\sE_{\ell_{0-1}}(h) - \sE_{\ell_{0-1}}(h^*)]^\alpha.
  \]
\end{restatable}

\begin{restatable}{lemma}{TsybakovEquiv}
\label{lemma:Tsybakov-equiv}
Assume that for any $h \in \sH_{\rm{all}}$, we have $\Pr[\hh(X) \neq
  \hh^*(X)] \leq c \E[\gamma(X) 1_{\gamma(X) \leq t
}]^{\alpha}$. Then, the Tsybakov noise condition holds, that is, there
exists a constant $B > 0$, such that
\begin{equation*}
\forall t > 0, \quad \Pr[\gamma(X) \leq t] \leq B t^{\frac{\alpha}{1 - \alpha}}.
\end{equation*}
\end{restatable}
The proofs of Lemma~\ref{lemma:Tsybakov} and
Lemma~\ref{lemma:Tsybakov-equiv} are included in
Appendix~\ref{app:Tsybakov}. To the best of our knowledge, there are
no previous results that formally analyze these properties of the
Tsybakov noise in the general multi-class classification setting,
although the similar result in the binary setting is well-known.
Next, we assume
that there exists a universal constant $c > 0$ and $\alpha \in [0, 1)$
  such that for all $h \in \sH$, the following Tsybakov noise
  inequality holds:
\begin{equation}
\label{eq:Tsybakov}
\E[1_{\hh(x) \neq \hh^*(x)}]
\leq c \bracket*{\sE_{\ell_{0-1}}(h) - \sE_{\ell_{0-1}}(h^*)}^\alpha.
\end{equation}
where $h^*$ is the Bayes-classifier.  We also assume that there is no
approximation error and that $\sM_{\ell_{0-1}}(\sH) = 0$. We refer to
this as the \emph{Tsybakov noise assumption} in multi-class
classification.

\begin{restatable}{theorem}{TsybakovMulti}
  \label{Thm:tsybakov-multi}
  Consider a multi-class classification setting where the Tsybakov
  noise assumption holds. Assume that the following holds for all $h
  \in \sH$ and $x \in \sX$: $\Delta\sC_{\ell_{0-1},\sH}(h, x) \leq
  \Gamma \paren*{\Delta \sC_{\ell,\sH}(h, x)}$, with $\Gamma(x) =
  x^{\frac{1}{s}}$, for some $s \geq 1$.  Then, for any $h \in \sH$,
  \[
  \sE_{\ell_{0-1}}(h) - \sE^*_{\ell_{0-1}}(\sH)
  \leq c^{\frac{s - 1}{s - \alpha(s - 1)}}\bracket*{\sE_{\ell}(h) - \sE^*_{\ell}(\sH) + \sM_{\ell}(\sH)}^{\frac{1}{s - \alpha(s - 1)}}.
  \]
\end{restatable}
To our knowledge, these are the first multi-class classification
$\sH$-consistency bounds, and
even excess error bounds (a special case where $\sH = \sH_{\rm{all}}$)
established under the Tsybakov noise assumption.  Here too, this
theorem offers a significantly improved $\sH$-consistency guarantee
for multi-class classification. For smooth loss functions, standard
$\sH$-consistency bounds rely on a square-root dependence ($s =
2$). This dependence is improved to a linear rate when the Massart
noise condition holds ($\alpha \to 1$), or to an intermediate rate
between linear and square-root for other values of $\alpha$ within the
range $(0, 1)$. The proof is given in
Appendix~\ref{app:tsybakov-multi}.  Illustrative examples of these
bounds for constrained losses and comp-sum losses are presented in
Tables~\ref{tab:example-multi-cstnd} and \ref{tab:example-multi-comp}.

\begin{table*}[t]
\vskip -.2in
  \centering
  \resizebox{\textwidth}{!}{
  \begin{tabular}{@{\hspace{0cm}}llll@{\hspace{0cm}}}
    \toprule
   Loss functions & $\ell$  & $\Gamma$  & $\sH$-consistency bounds\\
    \midrule
    Constrained hinge & $\sum_{y'\neq y} \Phi_{\mathrm{hinge}} \paren*{-h(x, y')}$ & $x^1$ & $ \sE_{\ell}(h) - \sE^*_{\ell}(\sH) + \sM_{\ell}(\sH)$\\
    Constrained exponential & $\sum_{y'\neq y} \Phi_{\mathrm{exp}} \paren*{-h(x, y')}$    & $x^2$  & $ c^{\frac{1}{2 - \alpha}}\bracket*{\sE_{\ell}(h) - \sE^*_{\ell}(\sH) + \sM_{\ell}(\sH)}^{\frac{1}{2 - \alpha}}$  \\
    Constrained squared-hinge  & $\sum_{y'\neq y} \Phi_{\mathrm{sq-hinge}} \paren*{-h(x, y')}$ & $x^2$ &    $ c^{\frac{1}{2 - \alpha}}\bracket*{\sE_{\ell}(h) - \sE^*_{\ell}(\sH) + \sM_{\ell}(\sH)}^{\frac{1}{2 - \alpha}}$    \\
    Constrained $\rho$-margin & $\sum_{y'\neq y} \Phi_{\rho} \paren*{-h(x, y')}$ & $x^1$ &  $ \sE_{\ell}(h) - \sE^*_{\ell}(\sH) + \sM_{\ell}(\sH)$\\
    \bottomrule
  \end{tabular}
  }
  \caption{Examples of enhanced $\sH$-consistency bounds under the Tsybakov noise assumption and with symmetric and complete hypothesis sets, as provided by
  Theorem~\ref{Thm:tsybakov-multi}, for constrained losses $\ell(h, x, y) = \Phi^{\rm{cstnd}}(h, x, y) = \sum_{y'\neq y} \Phi \paren*{-h(x, y')} \text{ subject to } \sum_{y\in \sY} h(x, y) = 0$ (with only the surrogate portion
  displayed).}
\label{tab:example-multi-cstnd}
\end{table*}

\begin{table*}[t]
  \centering
  \resizebox{\textwidth}{!}{
  \begin{tabular}{@{\hspace{0cm}}llll@{\hspace{0cm}}}
    \toprule
   Loss functions & $\ell$  & $\Gamma$  & $\sH$-consistency bounds\\
    \midrule
    Sum exponential & $\sum_{y'\neq y} e^{h(x, y') - h(x, y)} $ & $x^1$ & $ \sE_{\ell}(h) - \sE^*_{\ell}(\sH) + \sM_{\ell}(\sH)$\\
    Multinomial logistic & $-\log \paren*{\frac{e^{h(x, y)}}{\sum_{y' \in\sY}e^{h(x, y')}}}$    & $x^2$  & $ c^{\frac{1}{2 - \alpha}}\bracket*{\sE_{\ell}(h) - \sE^*_{\ell}(\sH) + \sM_{\ell}(\sH)}^{\frac{1}{2 - \alpha}}$  \\
    Generalized cross-entropy  & $\frac{1}{\alpha} \bracket*{1 - \bracket*{\frac{e^{h(x, y)}}{\sum_{y'\in \sY} e^{h(x, y')}}}^{\alpha}}$ & $x^2$ &   $ c^{\frac{1}{2 - \alpha}}\bracket*{\sE_{\ell}(h) - \sE^*_{\ell}(\sH) + \sM_{\ell}(\sH)}^{\frac{1}{2 - \alpha}}$    \\
    Mean absolute error & $ 1 - \frac{e^{h(x, y)}}{\sum_{y'\in \sY} e^{h(x, y')}}$ & $x^1$ &  $ \sE_{\ell}(h) - \sE^*_{\ell}(\sH) + \sM_{\ell}(\sH)$\\
    \bottomrule
  \end{tabular}
  }
  \caption{Examples of enhanced $\sH$-consistency bounds under the Tsybakov noise assumption and with symmetric and complete hypothesis sets, as provided by
  Theorem~\ref{Thm:tsybakov-multi}, for comp-sum losses (with only the surrogate portion
  displayed).}
\label{tab:example-multi-comp}
\end{table*}

\ignore{
It is worth noting that deriving $\sH$-consistency
bounds for any concave $\Gamma$ under low-noise conditions in both
binary and multi-class classification remains a potential avenue for
future work.
}

\section{Bipartite ranking}
\label{sec:bi-ranking}

In preceding sections, we demonstrated how our new tools
enable the derivation of enhanced $\sH$-consistency bounds in various
classification scenarios: standard multi-class classification and
low-noise regimes of both binary and multi-class classification. Here,
we extend the applicability of our refined tools to the bipartite
ranking setting. We illustrate how they facilitate the establishment
of more favorable $\sH$-consistency bounds for classification
surrogate losses $\ell_{\Phi}$ with respect to the bipartite ranking
surrogate losses $\sfL_{\Phi}$. The loss functions $\ell_{\Phi}$
and $\sfL_{\Phi}$ are defined as follows:
\begin{align*}
  \ell_{\Phi}(h, x, y) = \Phi(yh(x)), \quad \sfL_{\Phi}(h, x, x', y, y')
  = \Phi \paren[\Big]{\tfrac{(y - y') \paren*{h(x) - h(x')}}{2}} 1_{y \neq y'},
\end{align*}
where $\Phi$ is a non-negative and non-increasing function. We will
say that $\ell_{\Phi}$ admits an $\sH$-consistency bound with respect
to $\sfL_{\Phi}$, if there exists a concave function $\Gamma \colon
\Rset_{+} \to \Rset_{+}$ with $\Gamma(0) = 0$ such that the following
inequality holds:
\begin{equation*}
  \sE_{\sfL_{\Phi}}(h) - \sE^*_{\sfL_{\Phi}}(\sH) +
  \sM_{\sfL_{\Phi}}(\sH)\leq \Gamma\paren*{\sE_{\ell_{\Phi}}(h)
    - \sE^*_{\ell_{\Phi}}(\sH) + \sM_{\ell_{\Phi}}(\sH)},
\end{equation*}
where $\sM_{\ell_{\Phi}}(\sH) = \sE^*_{\ell_{\Phi}}(\sH) - \E_X
\bracket*{\sC^*_{\ell_{\Phi}}(\sH, x)} $ and $\sM_{\sfL_{\Phi}}(\sH) =
\sE^*_{\sfL_{\Phi}}(\sH) - \E_{(x, x')} \bracket*{\ov
  \sC^*_{\sfL_{\Phi}}(\sH, x,x')} $ represent the minimizability gaps
for $\ell_{\Phi}$ and $\sfL_{\Phi}$, respectively.

\subsection{Fundamental tools for bipartite ranking}
\label{sec:bi-tools}

We first extend our new fundamental tools to the bipartite ranking setting.

\begin{restatable}{theorem}{NewBoundConcaveRanking}
\label{Thm:new-bound-concave-ranking}
Assume that there exist two concave functions $\Gamma_1 \colon \Rset_+
\to \Rset$ and $\Gamma_2 \colon \Rset_+ \to \Rset$, and two positive
functions $\alpha_1 \colon \sH \times \sX \to \Rset^*_+$ and $\alpha_2
\colon \sH \times \sX \to \Rset^*_+$ with $\E_{x \in \sX}
\bracket*{\alpha_1(h, x)} < +\infty$ and $\E_{x \in \sX}
\bracket*{\alpha_2(h, x)} < +\infty$ for all $h \in \sH$ such that the
following holds for all $h\in \sH$ and $(x, x') \in \sX \times \sX$:
$\Delta \ov \sC_{\lbi, \sH}(h, x, x') \leq \Gamma_1
\paren*{\alpha_1(h, x') \, \Delta \sC_{\ell,\sH}(h, x)} + \Gamma_2
\paren*{\alpha_2(h, x) \, \Delta \sC_{\ell,\sH}(h, x')}$. Then, the
following inequality holds for any hypothesis $h \in \sH$:
\begin{align*}
  \sE_{\lbi}(h) - \sE_{\lbi}^*(\sH) + \sM_{\lbi}(\sH)
  & \leq \Gamma_1 \paren*{ \gamma_1(h) \sfD_{\ell}(h)}
+ \Gamma_2 \paren*{\gamma_2(h) \sfD_{\ell}(h)}.
\end{align*}
\end{restatable}
with $\gamma_1(h) = \E_{x \in \sX} \bracket*{\alpha_1(h, x)}$,
$\gamma_2(h) = \E_{x \in \sX} \bracket*{\alpha_2(h, x)}$, and
$\sfD_{\ell}(h) = \sE_{\ell}(h) - \sE^*_{\ell}(\sH) + \sM_{\ell}(\sH)$.

The proof, detailed in Appendix~\ref{app:new-bound-concave-ranking},
leverages the fact that in the bipartite ranking setting, two pairs
$(x, y)$ and $(x', y')$ are drawn i.i.d.\ according to the
distribution $\sD$. As in the classification setting,
Theorem~\ref{Thm:new-bound-concave-ranking} is a fundamental tool for
establishing enhanced $\sH$-consistency bounds. This is achieved
incorporating the additional terms $\alpha_1$ and $\alpha_2$, which
can depend on both the hypothesis $h$ and the instances $x$ or $x'$,
thereby offering greater flexibility.

Note that such enhanced $\sH$-consistency bounds are meaningful only
when $\Gamma_1(0) + \Gamma_2(0) = 0$. This ensures that when the
minimizability gaps vanish (e.g., in the case where $\sH =
\sH_{\rm{all}}$ or in more generally realizable cases), the estimation
error of classification losses $\sE_{\ell}(h) - \sE^*_{\ell}(\sH)$ is
zero implies that the estimation error of bipartite ranking losses
$\sE_{\lbi}(h) - \sE_{\lbi}^*(\sH)$ is also zero.
This requires that there exists $\Gamma_1$ and $\Gamma_2$ such that
$\Gamma_1(0) + \Gamma_2(0) = 0$ and $\Delta \ov \sC_{\lbi, \sH}(h, x,
x') \leq \Gamma_1 \paren*{\alpha_1(h, x') \, \Delta \sC_{\ell,\sH}(h,
  x)} + \Gamma_2 \paren*{\alpha_2(h, x) \, \Delta \sC_{\ell,\sH}(h,
  x')}$, for all $h\in \sH$ and $(x, x') \in \sX \times \sX$. Note
that a necessary condition for this requirement is \emph{calibration}:
we say that a classification loss $\ell$ is \emph{calibrated} with
respect to a bipartite ranking loss $\sfL$, if for all $h\in
\sH_{\rm{all}}$ and $(x, x') \in \sX \times \sX$:
\begin{equation*}
  \Delta \sC_{\ell, \sH_{\rm{all}}}(h, x)
  = 0 \text{ and } \Delta \sC_{\ell, \sH_{\rm{all}}}(h, x')
  = 0 \implies \Delta \ov \sC_{\lbi, \sH_{\rm{all}}}(h, x, x') = 0.
\end{equation*}

We now introduce a family of auxiliary functions that are
differentiable and that admit a property facilitating the calibration
between $\ell_{\Phi}$ and $\sfL_{\Phi}$.

\begin{restatable}{theorem}{NewBoundConcaveRankingGeneral}
\label{Thm:new-bound-concave-ranking-general}
Assume that $\Phi$ is convex and differentiable, and satisfies 
$\Phi'(t) < 0$ for all $t \in \Rset$, and $\frac{\Phi'(t)}{\Phi'(-t)}
= e^{-\nu t}$ for some $\nu > 0$. Then, $\ell_{\Phi}$ is calibrated
with respect to $\sfL_{\Phi}$.
\end{restatable}

The proof can be found in
Appendix~\ref{app:new-bound-concave-ranking-general}.
Theorem~\ref{Thm:new-bound-concave-ranking-general} identifies a
family of functions $\Phi$ for which $\ell_{\Phi}$ is calibrated with
respect to $\sfL_{\Phi}$. This inclues the exponential loss and the
logistic loss, which fulfill the properties outlined in
Theorem~\ref{Thm:new-bound-concave-ranking-general}.
For the exponential loss, $\Phi(u) = \Phi_{\rm{exp}}(u) = e^{-u}$, we
have $ \frac{\Phi'_{\rm{exp}}(t)}{\Phi'_{\rm{exp}}(-t)} =
\frac{-e^{-t}}{-e^{t}} = e^{-2t}$. Similarly, for the logistic loss,
$\Phi(u) = \Phi_{\rm{log}}(u) = \log(1 + e^{-u})$, we have $
\frac{\Phi'_{\rm{log}}(t)}{\Phi'_{\rm{log}}(-t)} =
\frac{-\frac{1}{e^{t} + 1}}{-\frac{1}{e^{-t} + 1}} = e^{-t} $.  In the
next sections, we will prove $\sH$-consistency bounds in these two
cases.

\subsection{Exponential loss}
\label{sec:bi-exp}

We first consider the exponential loss, where $\Phi(u) =
\Phi_{\rm{exp}}(u) = e^{-u}$.
In the bipartite ranking setting, a hypothesis set $\sH$ is said to be
\emph{complete} if for any $x \in \sX$, $\curl*{h(x) \colon h \in
  \sH}$ spans $\Rset$.

\begin{restatable}{theorem}{NewBoundConcaveRankingExp}
\label{Thm:new-bound-concave-ranking-exp}
Assume that $\sH$ is complete. Then, the following inequality
holds for the exponential loss $\Phi_{\mathrm{exp}}$:
\begin{align*}
\Delta \ov \sC_{\lbi_{\Phi_{\mathrm{exp}}}, \sH}(h, x, x') 
\leq \sC_{\ell_{\Phi_{\mathrm{exp}}}}(h, x') \,
\Delta \sC_{\ell_{\Phi_{\mathrm{exp}}},\sH}(h, x) + \sC_{\ell_{\Phi_{\mathrm{exp}}}}(h, x) \,
\Delta \sC_{\ell_{\Phi_{\mathrm{exp}}},\sH}(h, x').
\end{align*}
Additionally, for any hypothesis $h \in \sH$, we have
\begin{align*}
  \sE_{\lbi_{\Phi_{\mathrm{exp}}}}(h) - \sE_{\lbi_{\Phi_{\mathrm{exp}}}}^*(\sH)
  + \sM_{\lbi_{\Phi_{\mathrm{exp}}}}(\sH) \leq 2 \sE_{\ell_{\Phi_{\mathrm{exp}}}}(h) \,
  \paren*{\sE_{\ell_{\Phi_{\mathrm{exp}}}}(h) - \sE^*_{\ell_{\Phi_{\mathrm{exp}}}}(\sH)
    + \sM_{\ell_{\Phi_{\mathrm{exp}}}}(\sH)}.
\end{align*}
\end{restatable}
See Appendix~\ref{app:new-bound-concave-ranking-exp} for the proof.
The proof leverages our new tool,
Theorem~\ref{Thm:new-bound-concave-ranking}, in conjunction with the
specific form of the conditional regrets for the exponential function
and the convexity of squared function.
This result is remarkable since it directly bounds the estimation
error of the RankBoost loss function by that of AdaBoost. The
observation that AdaBoost often exhibits favorable ranking accuracy,
often approaching that of RankBoost, was first highlighted by
\citet{cortes2003auc}. Later, \citet{rudin2005margin} introduced a
coordinate descent version of RankBoost and demonstrated that, when
incorporating the constant weak classifier, AdaBoost asymptotically
achieves the same ranking accuracy as coordinate descent RankBoost.

Here, we present a stronger non-asymptotic result for the
estimation losses of these algorithms. We show that when the
estimation error of the AdaBoost predictor $h$ is reduced to $\e$, the
corresponding RankBoost loss is bounded by $2
\sE_{\ell_{\Phi_{\mathrm{exp}}}}(h) \, \paren*{\e +
  \sM_{\ell_{\Phi_{\mathrm{exp}}}}(\sH)} -
\sM_{\lbi_{\Phi_{\mathrm{exp}}}}(\sH)$. This provides a stronger
guarantee for the ranking quality of AdaBoost. In the nearly
realizable case, where minimizability gaps are negligible, this upper
bound approximates to $2 \sE_{\ell_{\Phi_{\mathrm{exp}}}}(h) \e$, aligning
with the results of \citet{gao2015consistency} for excess errors,
where $\sH$ is assumed to be the family of all measurable functions.

\subsection{Logistic loss}
\label{sec:bi-log}

Here, we consider the logistic loss, where $\Phi(u) =
\Phi_{\rm{log}}(u) = \log(1 + e^{-u})$.

\begin{restatable}{theorem}{NewBoundConcaveRankingLog}
\label{Thm:new-bound-concave-ranking-log}
Assume that $\sH$ is complete. For any $x$, define $u(x) =
\max\curl*{\eta(x), 1 - \eta(x)}$. Then, the following inequality
holds for the logistic loss $\Phi_{\mathrm{log}}$:
\begin{align*}
\Delta \ov \sC_{\lbi_{\Phi_{\mathrm{log}}}, \sH}(h, x, x') 
& \leq u(x')
\Delta \sC_{\ell_{\Phi_{\mathrm{log}}}, \sH}(h, x)
+ u(x)
\Delta \sC_{\ell_{\Phi_{\mathrm{log}}}, \sH}(h, x').
\end{align*}
Furthermore, for any hypothesis $h \in \sH$, we have
\begin{align*}
  & \sE_{\lbi_{\Phi_{\mathrm{log}}}}(h) - \sE_{\lbi_{\Phi_{\mathrm{log}}}}^*(\sH)
  + \sM_{\lbi_{\Phi_{\mathrm{log}}}}(\sH)
\leq 2\E[u(X)]  \,
\paren*{\sE_{\ell_{\Phi_{\mathrm{log}}}}(h) - \sE^*_{\ell_{\Phi_{\mathrm{log}}}}(\sH)
  + \sM_{\ell_{\Phi_{\mathrm{log}}}}(\sH)}.
\end{align*}
\end{restatable}
Note that the term $\E[u(X)]$ can be expressed as $1 -
\E[\min\curl*{\eta(X), (1 - \eta(X))}]$, and coincides with the
accuracy of the Bayes classifier. In particular, in the deterministic
case, we have $\E[u(X)] = 1$.
The proof is given in
Appendix~\ref{app:new-bound-concave-ranking-log}. In the first part of
the proof, we establish and leverage the sub-additivity of
$\Phi_{\log}$: $\Phi_{\mathrm{log}}(h - h') \leq
\Phi_{\mathrm{log}}(h) + \Phi_{\mathrm{log}}(-h')$, to derive an upper
bound for $\Delta \ov \sC_{\lbi_{\Phi_{\mathrm{log}}}, \sH}(h, x, x')
$ in terms of $\Delta \sC_{\ell_{\Phi_{\mathrm{log}}}, \sH}(h, x)$ and
$\Delta \sC_{\ell_{\Phi_{\mathrm{log}}}, \sH}(h, x')$. Next, we apply
our new tool, Theorem~\ref{Thm:new-bound-concave-ranking}, with
$\alpha_1(h, x') = \max\curl*{\eta(x'), 1 - \eta(x')}$ and
$\alpha_2(h, x) = \max\curl*{\eta(x), 1 - \eta(x)}$.

Both our result and its proof are entirely novel. Significantly, this
result implies a parallel finding for logistic regression analogous to
that of AdaBoost: If $h$ is the predictor obtained by minimizing the
logistic loss estimation error to $\e$, then the
$\lbi_{\Phi_{\mathrm{log}}}$-estimation loss of $h$ for ranking is
bounded above by $2 \E[u(X)] (\e +
\sM_{\ell_{\Phi_{\mathrm{log}}}}(\sH)) -
\sM_{\lbi_{\Phi_{\mathrm{log}}}}(\sH)$. When minimizability gaps are
small, such as in realizable cases, this bound further simplifies to
$2 \E[u(X)] \e$, suggesting a favorable ranking property for logistic
regression.

This result is surprising, as the favorable ranking property of
AdaBoost and its connection to RankBoost were thought to stem from the
specific properties of the exponential loss, particularly its morphism
property, which directly links the loss functions of AdaBoost and
RankBoost. This direct connection does not exist for the logistic
loss, making our proof and result particularly remarkable.
In both cases, our new tools facilitated the derivation of
non-trivial inequalities where the factor plays a crucial role. The
exploration of enhanced $\sH$-consistency bounds for other functions
$\Phi$ is an interesting question for future research that we have
initiated. In the next section, we prove negative results for the
hinge loss.

\subsection{Hinge loss}
\label{sec:bi-hinge}

\ignore{Here, we show that when $\Phi$ is the hinge loss,
$\ell_{\Phi_{\rm{hinge}}}$ is not calibrated with respect to
$\sfL_{\Phi_{\rm{hinge}}}$, and that there are no meaningful
$\sH$-consistency bounds.}
The hinge loss $\ell_{\Phi_{\rm{hinge}}}$ is the loss function
minimized by the support vector machines (SVM)
\citep{cortes1995support} and $\sfL_{\Phi_{\rm{hinge}}}$ is the loss
function optimized by the RankSVM algorithm \citep{Joachims2002}
\ignore{(which in fact coincides with SVM
  \citep{MohriRostamizadehTalwalkar2018})}. However, the relationships
observed for AdaBoost and RankBoost, or Logistic Regression and its
ranking counterpart, do not hold here.  Instead, we present the
following two negative results.

\begin{restatable}{theorem}{NewBoundConcaveRankingHingeCalibration}
\label{Thm:new-bound-concave-ranking-hinge-calibration}
For the hinge loss, $\ell_{\Phi_{\rm{hinge}}}$ is not calibrated with
respect to $\sfL_{\Phi_{\rm{hinge}}}$.
\end{restatable}

\begin{restatable}[Negative result for hinge losses]{theorem}{NewBoundConcaveRankingHinge}
\label{Thm:new-bound-concave-ranking-hinge}
Assume that $\sH$ contains the constant function $1$. For the hinge
loss, if there exists a function pair $(\Gamma_1, \Gamma_2)$ such that
the following holds for all $h\in \sH$ and $(x, x') \in \sX \times
\sX$, with some positive functions $\alpha_1 \colon \sH \times \sX \to
\Rset^*_+$ and $\alpha_2 \colon \sH \times \sX \to \Rset^*_+$:
\begin{equation*}
\Delta \ov \sC_{\lbi_{\Phi_{\mathrm{hinge}}}, \sH}(h, x, x') \leq \Gamma_1 \paren*{\alpha_1(h, x') \,
\Delta \sC_{\ell_{\Phi_{\mathrm{hinge}}}, \sH}(h, x)} + \Gamma_2 \paren*{\alpha_2(h, x) \,
\Delta \sC_{\ell_{\Phi_{\mathrm{hinge}}}, \sH}(h, x')},
\end{equation*}
then, we have $\Gamma_1(0) + \Gamma_2(0) \geq \frac{1}{2}$.
\end{restatable}
See Appendix~\ref{app:new-bound-concave-ranking-hinge} for the
proof. Theorem~\ref{Thm:new-bound-concave-ranking-hinge} implies that
there are no meaningful $\sH$-consistency bounds for
$\ell_{\Phi_{\rm{hinge}}}$ with respect to $\sfL_{\Phi_{\rm{hinge}}}$
with common hypothesis sets. In Appendix~\ref{app:generalization}, we show that all our derived
enhanced $\sH$-consistency bounds can be used to provide novel
enhanced generalization bounds in their respective settings.

\ignore{In
Appendix~\ref{app:discussion}, we present a discussion of the role of
non-constant factors in our analysis, the applicability and
significance of our tools, as well as the connection of our results to
existing bipartite ranking results.
}

\ignore{
A natural question arises: which
functions $\Phi$ should we consider instead? We show below that
differentiability can help.
}

\section{Discussion}
\label{app:discussion}

\textbf{Role of non-constant factors}. One advantage of our enhanced
$\sH$-consistency bounds is their ability to incorporate non-constant
factors that reflect both the data distribution and the predictor. For
example, in the bounds with respect to the exponential loss in
bipartite ranking, the non-constant factor is the generalization error
of the AdaBoost-loss predictor. This means that the rate of the bound
becomes more favorable as the predictor's performance approaches that
of the best-in-class predictor. The best rate depends on the data
distribution, as does the best-in-class generalization
error. Similarly, in the enhanced $\sH$-consistency bounds with
respect to the logistic loss in bipartite ranking, the non-constant
factor is the accuracy of the Bayes classifier. This means that the
rate of the bound depends on the data distribution. In particular, in
the deterministic case, the accuracy of the Bayes classifier is one.

\noindent \textbf{Applicability and significance of our new tools}.
Our new fundamental tools enable the derivation of more favorable
guarantees in various scenarios that (i) better leverage key
distributional properties; (ii) establish connections between existing
algorithms; and (iii) can lead to more favorable algorithms in other
scenarios. For (i), an example is our more favorable $\sH$-consistency
bounds under low-noise conditions for both binary and multi-class
classification problems, with a linear rate when the Massart noise
condition holds and an intermediate rate between linear and
square-root for other values of $\alpha$ under the Tsybakov noise
assumption. For (ii), an example of this is our enhanced
$\sH$-consistency bounds in bipartite ranking, which provide a
theoretical explanation for the empirical observation that AdaBoost
tends to perform well in ranking tasks. A similar insight holds for
the logistic loss. For (iii), our new tools can also be applied to
comp-sum losses, and the enhanced bounds may contribute to the
development of more robust adversarial algorithms. This is similar to
the improvements in adversarial robustness achieved in
\citep{mao2023cross}, which presents an interesting direction for
future research.

\noindent \textbf{Connection to existing bipartite ranking
  results}. Prior work, such as \citep{kotlowski2011bipartite} and
\citep{agarwal2014surrogate}, has established the Bayes-consistency of
several classification surrogate losses with respect to bipartite
misranking loss. Specifically, \citet{kotlowski2011bipartite} examined
the exponential and logistic surrogate losses, while
\citet{agarwal2014surrogate} explored a broader class of proper
(composite) classification surrogate losses.

In contrast, our work focuses on establishing enhanced
$\sH$-consistency bounds for classification surrogate losses with
respect to surrogate bipartite misranking losses. For instance, we
prove $\sH$-consistency bounds for the classification exponential loss
(used in AdaBoost) with respect to the bipartite misranking
exponential loss (used in RankBoost).  These are in some sense
stronger results since, combined with the standard consistency of
bipartite misranking surrogate losses with respect to the bipartite
misranking loss, our results imply the consistency of classification
surrogate losses with respect to the bipartite misranking loss.
Moreover, our contributions go beyond Bayes-consistency by providing
more informative $\sH$-consistency bounds. An interesting future
direction is to extend our enhanced $\sH$-consistency bounds to more
general strongly proper losses considered in
\citep{agarwal2014surrogate} with respect to their corresponding
bipartite misranking counterparts.

\noindent \textbf{Faster rates compared to previous work}.  Recent
work by \cite{MaoMohriZhong2024} shows that for any smooth surrogate
loss in binary and multi-class classification, $\Gamma(\epsilon)$
behaves as $\sqrt{\epsilon}$ near zero.
However, our analysis demonstrates that under specific noise
conditions (distributional assumptions), $\sH$-consistency bounds can
achieve significantly improved rates, even approaching near-linear
behavior in limiting cases.
It is important to emphasize that the bounds derived in
\citep{MaoMohriZhong2024} are not incorrect and do not contradict our
results.  They are worst-case results that hold universally for
\emph{any} distribution.  Specifically, their bounds rely on a fixed
convex function $\Psi$ or concave function $\Gamma$, which is
independent of both the distribution and the hypothesis.  In contrast,
our framework introduces the auxiliary functions $\alpha$ and $\beta$,
which enable the derivation of refined bounds incorporating a
non-constant factor $\gamma$ that adapts to both the data distribution
and the predictor $h$.

Remarkably, under Tsybakov
noise conditions, we can derive more favorable $\sH$-consistency
bounds (See Theorems~\ref{Thm:tsybakov-binary} and
\ref{Thm:tsybakov-multi}, and the subsequent discussion) with better
exponents. In the proof, we choose functions $\alpha$ and $\beta$ that
depend on the input, the predictor $h$, and the best predictor
$h^*$. For example, $\beta(h, x)$ measures the disagreement of $h$ and
$h^*$ on $x$, modulo a small constant $\epsilon$.

\section{Conclusion}
\label{sec:conclusion}

We introduced novel tools for deriving enhanced $\sH$-consistency
bounds in various learning settings, including multi-class
classification, low-noise regimes, and bipartite ranking. Remarkably,
we established substantially more favorable guarantees for several
settings and demonstrated unexpected connections between
classification and bipartite ranking performances for the exponential
and logistic losses. Our tools are likely to be useful in
the analysis of $\sH$-consistency bounds for a wide range of other
scenarios.


\bibliography{ehcb}

\newpage
\appendix

\renewcommand{\contentsname}{Contents of Appendix}
\tableofcontents
\addtocontents{toc}{\protect\setcounter{tocdepth}{4}} 
\clearpage

\section{Related work}
\label{app:related-work}

Bayes-consistency has been well studied in a wide range of learning scenarios, including binary classification \citep{Zhang2003, bartlett2006convexity,
  steinwart2007compare, MohriRostamizadehTalwalkar2018}, multi-class classification \citep{zhang2004statistical,
  tewari2007consistency, ramaswamy2012classification, ramaswamy2014consistency, narasimhan2015consistent, agarwal2015consistent, williamson2016composite,
  ramaswamy2016convex, chen2006consistency, chen2006consistency2, liu2007fisher,
  dogan2016unified, wang2020weston, wang2023classification, wang2024unified}, multi-label learning
\citep{gao2011consistency, dembczynski2012consistent,koyejo2015consistent,zhang2020convex}, learning with rejection
\citep{ramaswamy2015consistent,CortesDeSalvoMohri2016,
  CortesDeSalvoMohriBis2016, CortesDeSalvoMohri2023,NiCHS19,caogeneralizing}, learning to defer \citep{mozannar2020consistent,verma2022calibrated,cao2023defense}, ranking \citep{ravikumar2011ndcg, ramaswamy2013convex,agarwal2014surrogate,gao2015consistency, uematsu2017theoretically}, cost sensitive learning \citep{pires2013cost, pires2016multiclass}, structured prediction \citep{ciliberto2016consistent, osokin2017structured,
  blondel2019structured}, general embedding framework
\citep{finocchiaro2020embedding,frongillo2021surrogate,waggoner2021unifying,finocchiaro2022embedding,nueve2024trading}, Top-$k$
classification \citep{lapin2015top,yang2020consistency,thilagar2022consistent}, hierarchical classification \citep{ramaswamy2015convex,cao2024consistent}, ordinal regression \citep{pedregosa2017consistency}, and learning from noisy labels \citep{natarajan2013learning,scott2013classification,menon2015learning,liu2015classification,patrini2016loss,liu2020peer,zhang2021learning,zhang2022learning,zhang2024multiclass}.  However, this classical notion has
significant limitations since it only holds asymptotically and for the
impractical set of all measurable functions. Thus, it fails to provide
guarantees for real-world scenarios where learning is restricted to
specific hypothesis sets, such as linear models or neural networks.
In fact, Bayes-consistency does not always translate into superior
performance, as highlighted by \citet{long2013consistency} (see also \citep{zhang2020bayes}).

\citet{awasthi2022Hconsistency} proposed the key notion of
$\sH$-consistency bounds for binary classification. These novel
non-asymptotic learning guarantees for binary classification account
for the hypothesis set $\sH$ adopted and are more significant and
informative than existing Bayes-consistency guarantees. They provided
general tools for deriving such bounds and used them to establish a
series of $\sH$-consistency bounds in both standard binary
classification and binary classification under Massart's noise
condition. \citet{awasthi2022multi} and \citet{zheng2023revisiting}
further generalized those general tools to standard multi-class
classification and used them to establish multi-class
$\sH$-consistency bounds. Specifically, \citet{awasthi2022multi}
presented a comprehensive analysis of $\sH$-consistency bounds for the
three most commonly used families of multi-class surrogate losses:
\emph{max losses} \citep{crammer2001algorithmic}, \emph{sum losses}
\citep{weston1998multi}, and \emph{constrained losses}
\citep{lee2004multicategory}. They showed negative results for max
losses, while providing positive results for sum losses and
constrained losses. Additionally, \citet{zheng2023revisiting} used
these general tools in multi-class classification to derive
$\sH$-consistency bounds for the (multinomial) logistic loss
\citep{Verhulst1838,Verhulst1845,Berkson1944,Berkson1951}. Meanwhile,
\citet{mao2023cross} presented a theoretical analysis of
$\sH$-consistency bounds for a broader family of loss functions,
termed \emph{comp-sum losses}, which includes sum losses and
cross-entropy (or logistic loss) as special cases, and also includes
generalized cross-entropy \citep{zhang2018generalized}, mean absolute
error \citep{ghosh2017robust}, and other cross-entropy-like loss
functions. In all these works, determining whether $\sH$-consistency
bounds hold and deriving these bounds have required specific proofs
and analyses for each surrogate
loss. \citet{MaoMohriZhong2023characterization} complemented these
efforts by providing both a general characterization and an extension
of $\sH$-consistency bounds for multi-class classification, based on
the error transformation functions they defined for comp-sum losses
and constrained losses. Recently, \citet{MaoMohriZhong2024} further
applied these error transformations to characterize the general
behavior of these bounds, showing that the universal growth rate of
$\sH$-consistency bounds for smooth surrogate losses in both binary
and multi-class classification is square-root.  $\sH$-consistency
bounds have also been studied in other learning scenarios including
pairwise ranking \citep{MaoMohriZhong2023ranking,MaoMohriZhong2023rankingabs}, learning with
rejection
\citep{MaoMohriZhong2024score,MaoMohriZhong2024predictor,MohriAndorChoiCollinsMaoZhong2024learning},
learning to defer
\citep{MaoMohriMohriZhong2023twostage,MaoMohriZhong2024deferral,mao2024regression,mao2024realizable,MaoMohriZhong2025mastering,mao2025theory,desalvo2025budgeted}, top-$k$ classification \citep{cortes2024cardinality}, imbalanced learning \citep{cortes2025balancing,cortes2025improved}, optimization of generalized metrics \citep{MaoMohriZhong2025principled},
adversarial robustness \citep{awasthi2021calibration,awasthi2021finer,AwasthiMaoMohriZhong2023theoretically,awasthi2023dc,mao2023cross}, multi-label learning \citep{mao2024multi}, bounded regression \citep{mao2024h}, learning under the Model Margin (MM) low-noise condition \citep{mohri2025beyond}, and structured prediction \citep{MaoMohriZhong2023structured}.

All previous bounds in
the aforementioned work were derived under the condition that a lower
bound of the surrogate loss conditional regret is given as a convex
function of the target conditional regret, without non-constant
factors depending on the predictor or input instance. In this work, we
relax this condition and present a general framework for establishing
enhanced $\sH$-consistency bounds based on more general inequalities
relating conditional regrets, leading to finer and more favorable
$\sH$-consistency bounds.

\newpage
\section{Proof of new fundamental tools (Theorem~\ref{Thm:new-bound-convex}, Theorem~\ref{Thm:new-bound-concave}, Theorem~\ref{Thm:new-bound-tight} and Theorem~\ref{Thm:new-bound-power})}
\label{app:general}

\NewBoundConvex*
\begin{proof}
For any $h\in \sH$, we can write
\begin{align*}
  \Psi\paren*{\sE_{\ell_2}(h) - \sE^*_{\ell_2}(\sH) + \sM_{\ell_2}(\sH)}
  & = \Psi\paren*{\E_{X} \bracket*{\Delta\sC_{\ell_2,\sH}(h, x)}}\\
  & = \Psi\paren*{\E_{X} \bracket*{\beta(h, x) \frac{\Delta\sC_{\ell_2,\sH}(h, x)}{\beta(h, x)}}}\\
  & \leq \E_{X} \bracket*{\beta(h, x) \Psi\paren*{\frac{\Delta\sC_{\ell_2,\sH}(h, x)}{\beta(h, x)}}}
  \tag{Jensen's ineq.} \\
  & \leq \E_{X} \bracket*{\alpha(h, x) \beta(h, x) \Delta\sC_{\ell_1,\sH}(h, x)}
  \tag{assumption}\\
  & \leq \bracket*{\sup_{x \in \sX} \alpha(h, x) \beta(h, x)} \E_X \bracket*{\Delta\sC_{\ell_1,\sH}(h, x)}
  \tag{H\"older’s ineq.}\\
  & = \bracket*{\sup_{x \in \sX} \alpha(h, x) \beta(h, x)} \paren*{\sE_{\ell_1}(h) - \sE^*_{\ell_1}(\sH) + \sM_{\ell_1}(\sH)}
  \tag{def. of $\E_X \bracket*{\Delta\sC_{\ell_1,\sH}(h, x)}$}.
\end{align*}
If, additionally, $\sX$ is a subset of $\Rset^n$ and, for any $h \in \sH$,
$x \mapsto \Delta\sC_{\ell_1,\sH}(h, x)$ is non-decreasing and $x \mapsto \alpha(h, x) \beta(h, x)$ is non-increasing, or vice-versa, then, by the FKG
inequality \citep{FortuinKasteleynGinibre1971}, we have
\begin{align*}
  \Psi\paren*{\sE_{\ell_2}(h) - \sE^*_{\ell_2}(\sH) + \sM_{\ell_2}(\sH)}
  & \leq \E_{X} \bracket*{\alpha(h, x) \beta(h, x)  \Delta\sC_{\ell_1,\sH}(h, x)}\\
  & \leq \E_{X} \bracket*{\alpha(h, x) \beta(h, x)}  \E_{X} \bracket*{\Delta\sC_{\ell_1,\sH}(h, x)}\\
  & \leq \E_{X} \bracket*{\alpha(h, x) \beta(h, x)}  \paren*{\sE_{\ell_1}(h) - \sE^*_{\ell_1}(\sH) + \sM_{\ell_1}(\sH)},
\end{align*}
which completes the proof.
\end{proof}

\NewBoundConcave*
\begin{proof}
For any $h\in \sH$, we can write
\begin{align*}
  \sE_{\ell_2}(h) - \sE^*_{\ell_2}(\sH) + \sM_{\ell_2}(\sH)
  & = \E_{X} \bracket*{\Delta\sC_{\ell_2,\sH}(h, x)}\\
  & \leq \E_{X} \bracket*{\beta(h, x) \Gamma \paren*{\alpha(h, x) \,\Delta \sC_{\ell_1,\sH}(h, x)}}
  \tag{assumption}\\
  & \leq \Gamma  \paren*{ \E_{X} \bracket*{\alpha(h, x) \beta(h, x) \Delta\sC_{\ell_1,\sH}(h, x)} }\\
  \tag{Jensen's ineq.} \\
  & \leq \Gamma \paren*{ \bracket*{\sup_{x \in \sX} \alpha(h, x) \beta(h, x)} \E_X \bracket*{\Delta\sC_{\ell_1,\sH}(h, x)} }
  \tag{H\"older’s ineq.}\\
  & = \Gamma \paren*{ \bracket[\Big]{\sup_{x \in \sX} \alpha(h, x) \beta(h, x)} \paren*{\sE_{\ell_1}(h) - \sE^*_{\ell_1}(\sH) + \sM_{\ell_1}(\sH)} }
  \tag{def. of $\E_X \bracket*{\Delta\sC_{\ell_1,\sH}(h, x)}$}.
\end{align*}
If, additionally, $\sX$ is a subset of $\Rset^n$ and, for any $h \in \sH$,
$x \mapsto \Delta\sC_{\ell_1,\sH}(h, x)$ is non-decreasing and $x \mapsto \alpha(h, x) \beta(h, x)$ is non-increasing, or vice-versa, then, by the FKG
inequality \citep{FortuinKasteleynGinibre1971}, we have
\begin{align*}
  \sE_{\ell_2}(h) - \sE^*_{\ell_2}(\sH) + \sM_{\ell_2}(\sH)
  & \leq \Gamma  \paren*{ \E_{X} \bracket*{\alpha(h, x) \beta(h, x) \Delta\sC_{\ell_1,\sH}(h, x)} }\\
  & \leq \Gamma \paren*{ \E_{X} \bracket*{\alpha(h, x) \beta(h, x)}  \E_{X} \bracket*{\Delta\sC_{\ell_1,\sH}(h, x)} }\\
  & \leq \Gamma \paren*{ \E_{X} \bracket*{\alpha(h, x) \beta(h, x)}  \paren*{\sE_{\ell_1}(h) - \sE^*_{\ell_1}(\sH) + \sM_{\ell_1}(\sH)} },
\end{align*}
which completes the proof.
\end{proof}

\NewBoundTight*
\begin{proof}
Take $h$ and $x$ such that $\sup_{x \in \sX} \alpha(h, x) \beta(h, x)$ is achieved. Consider the distribution concentrates on that $x$. Then, the bounds given by \eqref{eq:new-bound-convex} and \eqref{eq:new-bound-concave} reduce to the following forms:
\begin{align*}
\Psi\paren*{\frac{\Delta\sC_{\ell_2,\sH}(h,
    x)}{\beta(h, x)}} & \leq \alpha(h, x) \,
\Delta \sC_{\ell_1,\sH}(h, x)\\
  \frac{\Delta\sC_{\ell_2,\sH}(h,
    x)}{\beta(h, x)} & \leq \Gamma \paren*{\alpha(h, x) \,
\Delta \sC_{\ell_1,\sH}(h, x)}
\end{align*}
where we used the fact that $\E_X\bracket*{\beta(h, x)} = 1$ in this case.
They exactly match the assumptions in Theorems~\ref{Thm:new-bound-convex} and \ref{Thm:new-bound-concave}, which are the tightest inequalities that can be obtained. The $\sH$-consistency bound is in fact an equality in the same cases when the assumption holds with the best choices of $\alpha$ and $\beta$.
\end{proof}

\NewBoundPower*

\begin{proof}
For any $h\in \sH$, we can write
\begin{align*}
  \sE_{\ell_2}(h) - \sE^*_{\ell_2}(\sH) + \sM_{\ell_2}(\sH)
  & = \E_{X} \bracket*{\Delta\sC_{\ell_2,\sH}(h, x)}\\
  & \leq \E_{X} \bracket*{\beta(h, x) \alpha^{\frac{1}{s}}(h, x) \,\Delta \sC^{\frac{1}{s}}_{\ell_1,\sH}(h, x)}
  \tag{assumption}\\
  & \leq \E_X \bracket*{\alpha^{\frac{t}{s}}(h, x) \beta^t(h, x)}^{\frac{1}{t}} \E_X \bracket*{\Delta\sC_{\ell_1,\sH}(h, x)}^{\frac{1}{s}} 
  \tag{H\"older’s ineq.}\\
  & =  \E_X \bracket*{\alpha^{\frac{t}{s}}(h, x) \beta^t(h, x)}^{\frac{1}{t}} \paren*{\sE_{\ell_1}(h) - \sE^*_{\ell_1}(\sH) + \sM_{\ell_1}(\sH)}^{\frac{1}{s}}
  \tag{def. of $\E_X \bracket*{\Delta\sC_{\ell_1,\sH}(h, x)}$}.
\end{align*}
This completes the proof.
\end{proof}

\section{Proof of enhanced \texorpdfstring{$\sH$}{H}-consistency bounds in multi-class classification (Theorem~\ref{Thm:bound_lee_finer})}
\label{app:bound_lee_finer}

To begin the proof, we first introduce the following result from
\citet{awasthi2022multi}, which characterizes the conditional regret
of the multi-class zero-one loss. For completeness, we include the
proof here.

\begin{restatable}{lemma}{ExplicitAssumption}
\label{lemma:explicit_assumption_01}
Assume that $\sH$ is symmetric and complete. For any $x \in \sX$,
the best-in-class conditional error and
the conditional regret for $\ell_{0-1}$ can be expressed as follows:
\begin{align*}
\sC^*_{\ell_{0-1}, \sH}(x) & = 1 - \max_{y \in \sY} p(y \mid x)\\
\Delta \sC_{\ell_{0-1}, \sH}(h, x) & = \max_{y \in \sY} p(y \mid x) - p(\hh(x) \mid x)
\end{align*}
\end{restatable}
\begin{proof}
The conditional error for $\ell_{0-1}$ can be expressed as:
\begin{equation*}
\sC_{\ell_{0-1}}(h, x) = \sum_{y \in \sY} p(y \mid x) 1_{\hh(x) \neq y} = 1 - p(\hh(x) \mid y).
\end{equation*}
Since $\sH$ is symmetric and complete, we have $\curl*{\hh(x) \colon h \in \sH} = \sY$. Therefore, 
\begin{align*}
& \sC^*_{\ell_{0-1}, \sH}(x) = \inf_{h \in \sH} \sC_{\ell_{0-1}}(h, x) = 1 - \max_{y \in \sY} p(y \mid x)\\
& \Delta \sC_{\ell_{0-1}, \sH}(h, x) = \sC_{\ell_{0-1}}(h, x) - \sC^*_{\ell_{0-1}, \sH}(x) = \max_{y \in \sY} p(y \mid x) - p(\hh(x) \mid x).
\end{align*}
This completes the proof.
\end{proof}

\BoundLeeFiner*
\begin{proof}
The conditional error for constrained losses can be expressed as follows:
\begin{equation*}
\sC_{\Phi^{\mathrm{cstnd}}}(h, x) = \sum_{y \in \sY} p(y \mid x) \sum_{y'\neq y} \Phi \paren*{-h(x, y')} = \sum_{y \in \sY} \paren*{1 - p(y \mid x)} \Phi \paren*{-h(x, y)}.
\end{equation*}
Next, we will provide the proof for each case individually. We denote by $y_{\max}  = \argmax_{y \in \sY} p(y \mid x)$.

\textbf{Constrained exponential loss with $\Phi(u) = e^{-u}$.}
When $\hh(x) = y_{\max}$, we have $\Delta \sC_{\ell_{0-1}, \sH}(h, x) = 0$. Let $h \in \sH$ be a
hypothesis such that $\hh(x) \neq y_{\max}$. In this case, the conditional error can be written as:
\begin{equation*}
\begin{aligned}
 \sC_{\Phi^{\mathrm{cstnd}}}(h, x)
= \sum_{y\in \sY} \paren*{1 - p(y \mid x)} e^{h(x, y)} = \sum_{y\in \curl*{y_{\max}, \hh(x)}} \paren*{1 - p(y \mid x)} e^{h(x, y)} + \sum_{y \not \in \curl*{y_{\max}, \hh(x)}} e^{h(x, y)}.
\end{aligned}
\end{equation*}
For any $x\in \sX$, define the hypothesis $ h_{\mu} \in \sH$ by
\begin{align*}
 h_{\mu}(x, y) = 
\begin{cases}
  h(x, y) & \text{if $y \not \in \curl*{y_{\max}, \hh(x)}$}\\
  h(x, y_{\max}) + \mu & \text{if $y = \hh(x)$}\\
  h(x, \hh(x)) - \mu & \text{if $y = y_{\max}$}
\end{cases} 
\end{align*}
for any $\mu \in \Rset$. By the completeness of $\sH$, we have $ h_{\mu} \in \sH$ and  $\sum_{y \in \sY} h_{\mu}(x, y)=0$. Thus,
\begin{align*}
\Delta \sC_{\Phi^{\mathrm{cstnd}}, \sH}(h, x)
& = \sC_{\Phi^{\mathrm{cstnd}}}(h, x) - \sC^*_{\Phi^{\mathrm{cstnd}}}(\sH, x) \\
& \geq \sC_{\Phi^{\mathrm{cstnd}}}(h, x) - \inf_{\mu \in \Rset} \sC_{\Phi^{\mathrm{cstnd}}}(h_{\mu}, x)\\
& = \paren*{\sqrt{(1 - p(\hh(x) \mid x)) e^{h(x, \hh(x))}} - \sqrt{(1 - p(y_{\max} \mid x)) e^{h(x, y_{\max})}}}^2 \\
& \geq e^{h(x, \hh(x))} \paren*{\sqrt{(1 - p(\hh(x) \mid x))} - \sqrt{(1 - p(y_{\max} \mid x))}}^2  \tag{$e^{h(x, \hh(x))} \geq e^{h(x, y_{\max})}$ and $p(\hh(x) \mid x)\leq p(y_{\max} \mid x)$}\\
& = e^{h(x, \hh(x))} \paren*{\frac{p(y_{\max} \mid x) - p(\hh(x) \mid x)}{\sqrt{(1 - p(\hh(x) \mid x))} + \sqrt{(1 - p(y_{\max} \mid x))}}}^2\\
& \geq \frac{e^{h(x, \hh(x))}}{2} \paren*{\max_{y \in \sY} p(x, y) - p(\hh(x) \mid x)}^2 \tag{$0 \leq p(y_{\max} \mid x) + p(\hh(x) \mid x)\leq 1$}\\
& \geq \frac{e^{ \Lambda(h) }}{2} \paren*{ \Delta \sC_{\ell_{0-1}, \sH}(h, x)}^2.
\end{align*} 
Therefore, by Theorems~\ref{Thm:new-bound-convex} or \ref{Thm:new-bound-concave}, the following inequality holds
for any hypothesis $h \in \sH$:
\begin{align*}
\sE_{\ell_{0-1}}(h) - \sE_{\ell_{0-1}}^*(\sH) + \sM_{\ell_{0-1}}(\sH) \leq \frac{\sqrt{2}}{\paren*{{e^{\Lambda(h)}}}^{\frac{1}{2}}} \paren*{\sE_{\Phi^{\mathrm{cstnd}}}(h) - \sE_{\Phi^{\mathrm{cstnd}}}^*(\sH) + \sM_{\Phi^{\mathrm{cstnd}}}(\sH)}^{\frac{1}{2}}.
\end{align*}

\textbf{Constrained hinge loss with $\Phi(u) = \max \curl*{1 - u, 0}$.} When $\hh(x) = y_{\max}$, we have $\Delta \sC_{\ell_{0-1}, \sH}(h, x) = 0$. Let $h \in \sH$ be a
hypothesis such that $\hh(x) \neq y_{\max}$. In this case, the conditional error can be written as:
\begin{equation*}
\begin{aligned}
 \sC_{\Phi^{\mathrm{cstnd}}}(h, x)
& = \sum_{y\in \sY} \paren*{1 - p(y \mid x)} \max \curl*{1 + h(x, y), 0} \\
& = \sum_{y\in \curl*{y_{\max}, \hh(x)}} \paren*{1 - p(y \mid x)} \max \curl*{1 + h(x, y), 0} + \sum_{y \not \in \curl*{y_{\max}, \hh(x)}} \max \curl*{1 + h(x, y), 0}.
\end{aligned}
\end{equation*}
For any $x\in \sX$, define the hypothesis $ h_{\mu} \in \sH$ by
\begin{align*}
 h_{\mu}(x, y) = 
\begin{cases}
  h(x, y) & \text{if $y \not \in \curl*{y_{\max}, \hh(x)}$}\\
  h(x, y_{\max}) + \mu & \text{if $y = \hh(x)$}\\
  h(x, \hh(x)) - \mu & \text{if $y = y_{\max}$}
\end{cases} 
\end{align*}
for any $\mu \in \Rset$. By the completeness of $\sH$, we have $ h_{\mu} \in \sH$ and  $\sum_{y \in \sY} h_{\mu}(x, y)=0$. Thus,
\begin{align*}
\Delta \sC_{\Phi^{\mathrm{cstnd}}, \sH}(h, x)
& = \sC_{\Phi^{\mathrm{cstnd}}}(h, x) - \sC^*_{\Phi^{\mathrm{cstnd}}}(\sH, x) \\
& \geq \sC_{\Phi^{\mathrm{cstnd}}}(h, x) - \inf_{\mu \in \Rset} \sC_{\Phi^{\mathrm{cstnd}}}(h_{\mu}, x)\\
& \geq \paren*{1 + h(x, \hh(x))} \paren*{p(y_{\max} \mid x) - p(\hh(x) \mid x)}\\
& \geq \paren{1 + \Lambda(h)} \paren*{ \Delta \sC_{\ell_{0-1}, \sH}(h, x)}.
\end{align*} 
Therefore, by Theorems~\ref{Thm:new-bound-convex} or \ref{Thm:new-bound-concave}, the following inequality holds
for any hypothesis $h \in \sH$:
\begin{align*}
\sE_{\ell_{0-1}}(h) - \sE_{\ell_{0-1}}^*(\sH) + \sM_{\ell_{0-1}}(\sH) \leq \frac{1}{1 + \Lambda(h)} \paren*{\sE_{\Phi^{\mathrm{cstnd}}}(h) - \sE_{\Phi^{\mathrm{cstnd}}}^*(\sH) + \sM_{\Phi^{\mathrm{cstnd}}}(\sH)}.
\end{align*}

\textbf{Constrained squared hinge loss with $\Phi(u) = (1 - u)^2 1_{u \leq 1}$.} When $\hh(x) = y_{\max}$, we have $\Delta \sC_{\ell_{0-1}, \sH}(h, x) = 0$. Let $h \in \sH$ be a
hypothesis such that $\hh(x) \neq y_{\max}$. In this case, the conditional error can be written as:
\begin{equation*}
\begin{aligned}
 \sC_{\Phi^{\mathrm{cstnd}}}(h, x)
& = \sum_{y\in \sY} \paren*{1 - p(y \mid x)} \max \curl*{1 + h(x, y), 0}^2 \\
& = \sum_{y\in \curl*{y_{\max}, \hh(x)}} \paren*{1 - p(y \mid x)} \max \curl*{1 + h(x, y), 0}^2 + \sum_{y \not \in \curl*{y_{\max}, \hh(x)}} \max \curl*{1 + h(x, y), 0}^2.
\end{aligned}
\end{equation*}
For any $x\in \sX$, define the hypothesis $ h_{\mu} \in \sH$ by
\begin{align*}
 h_{\mu}(x, y) = 
\begin{cases}
  h(x, y) & \text{if $y \not \in \curl*{y_{\max}, \hh(x)}$}\\
  h(x, y_{\max}) + \mu & \text{if $y = \hh(x)$}\\
  h(x, \hh(x)) - \mu & \text{if $y = y_{\max}$}
\end{cases} 
\end{align*}
for any $\mu \in \Rset$. By the completeness of $\sH$, we have $ h_{\mu} \in \sH$ and  $\sum_{y \in \sY} h_{\mu}(x, y)=0$. Thus,
\begin{align*}
\Delta \sC_{\Phi^{\mathrm{cstnd}}, \sH}(h, x)
& = \sC_{\Phi^{\mathrm{cstnd}}}(h, x) - \sC^*_{\Phi^{\mathrm{cstnd}}}(\sH, x) \\
& \geq \sC_{\Phi^{\mathrm{cstnd}}}(h, x) - \inf_{\mu \in \Rset} \sC_{\Phi^{\mathrm{cstnd}}}(h_{\mu}, x)\\
& \geq \paren*{1 + h(x, \hh(x))}^2 \paren*{p(y_{\max} \mid x) - p(\hh(x) \mid x)}^2 \\
& \geq \paren{1 + \Lambda(h)}^2 \paren*{ \Delta \sC_{\ell_{0-1}, \sH}(h, x)}^2.
\end{align*} 
Therefore, by Theorems~\ref{Thm:new-bound-convex} or \ref{Thm:new-bound-concave}, the following inequality holds
for any hypothesis $h \in \sH$:
\begin{align*}
\sE_{\ell_{0-1}}(h) - \sE_{\ell_{0-1}}^*(\sH) + \sM_{\ell_{0-1}}(\sH) \leq \frac{1}{1 + \Lambda(h)} \paren*{\sE_{\Phi^{\mathrm{cstnd}}}(h) - \sE_{\Phi^{\mathrm{cstnd}}}^*(\sH) + \sM_{\Phi^{\mathrm{cstnd}}}(\sH)}^{\frac{1}{2}}.
\end{align*}

\end{proof}

\section{Proof of enhanced \texorpdfstring{$\sH$}{H}-consistency bounds under low-noise conditions}

\subsection{Proof of Theorem~\ref{Thm:tsybakov-binary}}
\label{app:tsybakov-binary}
\TsybakovBinary*
\begin{proof}
  Fix $\e > 0$ and define $\beta(h, x) = \frac{1_{h(x) \neq h^*(x)} + \e}{\E_X\bracket*{1_{h(x) \neq h^*(x)} + \e}}$.  Since $\Delta\sC_{\ell^{\rm{bi}}_{0-1},\sH}(h, x) = |2 \eta(x) - 1 |
  1_{h(x) \neq h^*(x)}$, we have $\frac{\Delta\sC_{\ell^{\rm{bi}}_{0-1},\sH}(h, x)}{\beta(h, x)} \leq \Delta\sC_{\ell^{\rm{bi}}_{0-1},\sH}(h, x)
  \E_X\bracket*{1_{h(x) \neq h^*(x)} + \e}$, thus the following inequality holds
\[
\frac{\Delta\sC_{\ell^{\rm{bi}}_{0-1},\sH}(h, x)}{\beta(h, x)}
\leq \E_X\bracket*{1_{h(x) \neq h^*(x)} + \e} \Delta \sC^{\frac{1}{s}}_{\ell,\sH}(h, x).
\]
By Theorem~\ref{Thm:new-bound-power}, with $\alpha(h, x)
= \E_X\bracket*{1_{h(x) \neq h^*(x)} + \e}^s$, we have
\begin{align*}
  \sE_{\ell^{\rm{bi}}_{0-1}}(h) - \sE^*_{\ell^{\rm{bi}}_{0-1}}(\sH)
  & \leq \paren*{\E_X\bracket*{1_{h(x) \neq h^*(x)} + \e}^t}^{\frac{1}{t}} \paren*{\sE_{\ell}(h) - \sE^*_{\ell}(\sH) + \sM_{\ell}(\sH)}^{\frac{1}{s}}.
\end{align*}
Since the inequality holds for any $\e > 0$, it implies:
\begin{align*}
  \sE_{\ell^{\rm{bi}}_{0-1}}(h) - \sE^*_{\ell^{\rm{bi}}_{0-1}}(\sH)
  & \leq \E_X\bracket*{\paren*{1_{h(x) \neq h^*(x)}}^t}^{\frac{1}{t}} \paren*{\sE_{\ell}(h) - \sE^*_{\ell}(\sH) + \sM_{\ell}(\sH)}^{\frac{1}{s}}\\
  & = \E_X[1_{h(x) \neq h^*(x)}]^{\frac{1}{t}} \paren*{\sE_{\ell}(h) - \sE^*_{\ell}(\sH) + \sM_{\ell}(\sH)}^{\frac{1}{s}} \tag{$\paren*{1_{h(x) \neq h^*(x)}}^t = 1_{h(x) \neq h^*(x)}$}\\
  & \leq c^{\frac{1}{t}} \bracket*{\sE_{\ell^{\rm{bi}}_{0-1}}(h) - \sE^*_{\ell^{\rm{bi}}_{0-1}}(\sH)}^{\frac{\alpha}{t}} \paren*{\sE_{\ell}(h) - \sE^*_{\ell}(\sH) + \sM_{\ell}(\sH)}^{\frac{1}{s}}
  \tag{Tsybakov noise assumption}
\end{align*}
The result follows after dividing both sides by $\bracket*{\sE_{\ell^{\rm{bi}}_{0-1}}(h) - \sE^*_{\ell^{\rm{bi}}_{0-1}}(\sH)}^{\frac{\alpha}{t}}$.
\end{proof}

\subsection{Proof of Lemma~\ref{lemma:Tsybakov} and Lemma~\ref{lemma:Tsybakov-equiv}}
\label{app:Tsybakov}

\Tsybakov*
\begin{proof}
  We prove the first inequality, the second one follows immediately the definition
  of the margin.
  By definition of the expectation and the Lebesgue integral, for
  any $u > 0$, we can write
  \begin{align*}
    \E[\gamma(X) 1_{\hh(X) \neq \hh^*(X)}]
    & = \int_0^{+\infty} \Pr[\gamma(X) 1_{\hh(X) \neq \hh^*(X)} > t] \, dt\\
    & \geq \int_0^{u} \Pr[\gamma(X) 1_{\hh(X) \neq \hh^*(X)} > t] \, dt\\
    & = \int_0^{u} \E[1_{\gamma(X) > t} \, 1_{\hh(X) \neq \hh^*(X)}] \, dt\\
    & = \int_0^{u} \paren*{\E[1_{\gamma(X) > t}] - \E[1_{\gamma(X) > t} 1_{\hh(X) = \hh^*(X)}]} \, dt\\
    & \geq \int_0^{u} \paren*{\E[1_{\gamma(X) > t}] - \E[1_{\hh(X) = \hh^*(X)}]} \, dt\\
    & = \int_0^{u} \paren*{\Pr[\hh(X) \neq \hh^*(X)] - \Pr[\gamma(X) \leq t]} \, dt\\
    & \geq u \Pr[\hh(X) \neq \hh^*(X)] - \int_0^{u} B t^{\frac{\alpha}{1 - \alpha}}  \, dt\\
    & = u \Pr[\hh(X) \neq \hh^*(X)] - B (1 - \alpha) u^{\frac{1}{1 - \alpha}}.
  \end{align*}
 Taking the derivative and choosing $u$ to maximize the above gives $u = \bracket*{\frac{\Pr[\hh(X) \neq \hh^*(X)]}{B}}^{\frac{1- \alpha}{\alpha}}$. Plugging in this choice of $u$ gives
  \[
  \E[\gamma(X) 1_{\hh(X) \neq \hh^*(X)}] \geq \bracket*{\frac{1}{B}}^{\frac{1 - \alpha}{\alpha}} \alpha \Pr[\hh(X) \neq \hh^*(X)]^{\frac{1}{\alpha}},
  \]
which can be rewritten as
$\Pr[\hh(X) \neq \hh^*(X)] \leq c \E[\gamma(X) 1_{\hh(X) \neq \hh^*(X)}]^\alpha$
for $c = \frac{B^{1 - \alpha}}{\alpha^{\alpha}}$.
\end{proof}

\TsybakovEquiv*
\begin{proof}
  Fix $t > 0$ and consider the event $\curl*{\gamma(X) \leq t}$. Since $h$ can be chosen to be any measurable function, there exists $h$ such $1_{\gamma(X) \leq t} = 1_{\hh(X) \neq \hh^*(X)}$. In view of that, we can write
  \begin{align*}
    \Pr[\gamma(X) \leq t]
    = \E[1_{\gamma(X) \leq t}]
    \leq c \E[\gamma(X) 1_{\gamma(X) \leq t}]^\alpha
    \leq c t^\alpha \E[1_{\gamma(X) \leq t}]^\alpha.
  \end{align*}
Comparing the left- and right-hand sides gives immediately
\[
\Pr[\gamma(X) \leq t] \leq c^{\frac{1}{1 - \alpha}} \, t^{\frac{\alpha}{1 - \alpha}}.
\]
Choosing $B =  c^{\frac{1}{1 - \alpha}}$ completes the proof.
\end{proof}

\subsection{Proof of Theorem~\ref{Thm:tsybakov-multi}}
\label{app:tsybakov-multi}
\TsybakovMulti*
\begin{proof}
  Fix $\e > 0$ and define $\beta(h, x) = \frac{1_{\hh(x) \neq \hh^*(x)} + \e}{\E_X\bracket*{1_{\hh(X) \neq \hh^*(X)} + \e}}$.  By Lemma~\ref{lemma:explicit_assumption_01}, $\Delta\sC_{\ell_{0-1},\sH}(h, x) = \max_{y \in \sY}p(y \mid x) - p(\hh(x) \mid x) =  p(\hh^*(x) \mid x) - p(\hh(x) \mid x) $, we have \[\frac{\Delta\sC_{\ell_{0-1},\sH}(h, x)}{\beta(h, x)} \leq \Delta\sC_{\ell_{0-1},\sH}(h, x)
  \E_X\bracket*{1_{\hh(X) \neq \hh^*(X)} + \e},\] thus the following inequality holds
\[
\frac{\Delta\sC_{\ell_{0-1},\sH}(h, x)}{\beta(h, x)}
\leq \E_X\bracket*{1_{\hh(X) \neq \hh^*(X)} + \e} \Delta \sC^{\frac{1}{s}}_{\ell,\sH}(h, x).
\]
By Theorem~\ref{Thm:new-bound-power}, with $\alpha(h, x)
= \E_X\bracket*{1_{\hh(X) \neq \hh^*(X)} + \e}^s$, we have
\begin{align*}
  \sE_{\ell_{0-1}}(h) - \sE^*_{\ell_{0-1}}(\sH)
  & \leq \paren*{\E_X\bracket*{1_{\hh(X) \neq \hh^*(X)} + \e}^t}^{\frac{1}{t}} \paren*{\sE_{\ell}(h) - \sE^*_{\ell}(\sH) + \sM_{\ell}(\sH)}^{\frac{1}{s}}.
\end{align*}
Since the inequality holds for any $\e > 0$, it implies:
\begin{align*}
  \sE_{\ell_{0-1}}(h) - \sE^*_{\ell_{0-1}}(\sH)
  & \leq \E_X\bracket*{\paren*{1_{\hh(X) \neq \hh^*(X)}}^t}^{\frac{1}{t}} \paren*{\sE_{\ell}(h) - \sE^*_{\ell}(\sH) + \sM_{\ell}(\sH)}^{\frac{1}{s}}\\
  & = \E_X[1_{\hh(X) \neq \hh^*(X)}]^{\frac{1}{t}} \paren*{\sE_{\ell}(h) - \sE^*_{\ell}(\sH) + \sM_{\ell}(\sH)}^{\frac{1}{s}} \tag{$\paren*{1_{\hh(X) \neq \hh^*(X)}}^t = 1_{\hh(x) \neq \hh^*(x)}$}\\
  & \leq c^{\frac{1}{t}} \bracket*{\sE_{\ell_{0-1}}(h) - \sE^*_{\ell_{0-1}}(\sH)}^{\frac{\alpha}{t}} \paren*{\sE_{\ell}(h) - \sE^*_{\ell}(\sH) + \sM_{\ell}(\sH)}^{\frac{1}{s}}
  \tag{Tsybakov noise assumption}
\end{align*}
The result follows after dividing both sides by $\bracket*{\sE_{\ell_{0-1}}(h) - \sE^*_{\ell_{0-1}}(\sH)}^{\frac{\alpha}{t}}$.
\end{proof}

\section{Proof of enhanced \texorpdfstring{$\sH$}{H}-consistency bounds in bipartite ranking}
\subsection{Proof of Theorem~\ref{Thm:new-bound-concave-ranking}}
\label{app:new-bound-concave-ranking}
\NewBoundConcaveRanking*
\begin{proof}
For any $h\in \sH$, we can write
\begin{align*}
  &\sE_{\lbi}(h) - \sE^*_{\lbi}(\sH) + \sM_{\lbi}(\sH)\\
  & = \E_{X,X'} \bracket*{\Delta \ov \sC_{\lbi,\sH}(h, x, x')}\\
  & \leq \E_{X,X'} \bracket*{\Gamma_1 \paren*{\alpha_1(h, x') \,
\Delta \sC_{\ell,\sH}(h, x)} + \Gamma_2 \paren*{\alpha_2(h, x) \,
\Delta \sC_{\ell,\sH}(h, x')}}
  \tag{assumption}\\
  & \leq \Gamma_1  \paren*{ \E_{X'} \bracket*{\alpha_1(h, x')}  \E_X\bracket*{  \Delta\sC_{\ell,\sH}(h, x)} } +  \Gamma_2  \paren*{ \E_{X} \bracket*{\alpha_2(h, x)}  \E_{X'}\bracket*{ \Delta\sC_{\ell,\sH}(h, x')} }\\
  \tag{Jensen's ineq.} \\
   & \leq \Gamma_1 \paren*{\E_{x \in \sX} \bracket*{\alpha_1(h, x)} \,
\paren*{\sE_{\ell}(h) - \sE^*_{\ell}(\sH) + \sM_{\ell}(\sH)}}\\
&\qquad + \Gamma_2 \paren*{\E_{x \in \sX} \bracket*{\alpha_2(h, x)} \,
\paren*{\sE_{\ell}(h) - \sE^*_{\ell}(\sH) + \sM_{\ell}(\sH)}}
  \tag{def. of $\E_X \bracket*{\Delta\sC_{\ell,\sH}(h, x)}$},
\end{align*}
which completes the proof.
\end{proof}

\subsection{Proof of Theorem~\ref{Thm:new-bound-concave-ranking-exp}}
\label{app:new-bound-concave-ranking-exp}

\NewBoundConcaveRankingExp*
\begin{proof}
To simplify the notation, we will drop the dependency on $x$. Specifically, we use $\eta$ to denote $\eta(x)$, $\eta'$ to denote $\eta(x')$, $h$ to denote $h(x)$, and $h'$ to denote $h(x')$. Thus, we can write:
\begin{align*}
   \Delta \ov \sC_{\lbi_{\Phi_{\mathrm{exp}}}, \sH}(h, x, x')
  & = \eta (1 - \eta') e^{-h + h'}
  +  \eta' (1 - \eta)  e^{-h' + h} - 2 \sqrt{\eta (1 - \eta') \eta' (1 - \eta)}\\
   \sC_{\ell_{\Phi_{\mathrm{exp}}}}(h, x)
  & = \eta  e^{-h} +  (1 - \eta) e^{h} \\
  \Delta \sC_{\ell_{\Phi_{\mathrm{exp}}},\sH}(h, x)
  & = \eta  e^{-h}
  +  (1 - \eta) e^{h} - 2 \sqrt{\eta (1 - \eta)}.
\end{align*}
For any $A, B \in \Rset$, we have $(A + B)^2 \leq 2(A^2 + B^2)$. To prove this inequality, observe that the function $x \mapsto x^2$ is convex. Therefore, we have:
\begin{equation*}
(A + B)^2 = 4 \paren*{\frac{A + B}{2}}^2 \leq 4 \paren*{\frac{A^2 + B^2}{2}} = 2 (A^2 + B^2).
\end{equation*}
In light of this inequality, we can write
\begin{align*}
  \Delta \ov \sC_{\lbi_{\Phi_{\mathrm{exp}}}, \sH}(h, x, x') 
  & = \paren*{\sqrt{\eta (1 - \eta') e^{-h + h'}} - \sqrt{ \eta' (1 - \eta)  e^{-h' + h}}}^2\\
  &  = \paren*{\sqrt{(1 - \eta') e^{h'}} \paren*{\sqrt{\eta  e^{-h}} - \sqrt{(1 - \eta) e^{h}}} + \sqrt{(1 - \eta) e^{h}} \paren*{\sqrt{(1 - \eta') e^{h'}} - \sqrt{\eta'  e^{-h'}}}}^2\\
  & \leq 2\paren*{(1 - \eta') e^{h'}} \paren*{\sqrt{\eta  e^{-h}} - \sqrt{(1 - \eta) e^{h}}}^2\\
  & \qquad + 2\paren*{(1 - \eta) e^{h}} \paren*{\sqrt{\eta'  e^{-h'}} - \sqrt{(1 - \eta') e^{h'}}}^2 \tag{$(A + B)^2 \leq 2 (A^2 + B^2)$}\\
  \Delta \ov \sC_{\lbi_{\Phi_{\mathrm{exp}}}, \sH}(h, x, x')
  & = \paren*{\sqrt{ \eta' (1 - \eta)  e^{-h' + h}} - \sqrt{\eta (1 - \eta') e^{-h + h'}} }^2\\
  & = \paren*{\sqrt{\eta'  e^{-h'}} 
  \paren*{\sqrt{(1 - \eta) e^{h}} - \sqrt{\eta  e^{-h}}} + \sqrt{\eta  e^{-h}} \paren*{\sqrt{\eta'  e^{-h'}} - \sqrt{(1 - \eta') e^{h'}}}}^2\\
  & \leq 2\paren*{\eta'  e^{-h'}} \paren*{\sqrt{\eta  e^{-h}} - \sqrt{(1 - \eta) e^{h}}}^2\\
  & \qquad + 2\paren*{\eta  e^{-h}} \paren*{\sqrt{\eta'  e^{-h'}} - \sqrt{(1 - \eta') e^{h'}}}^2 \tag{$(A + B)^2 \leq 2 (A^2 + B^2)$}.
\end{align*}
Thus, by taking the mean of the two inequalities, we obtain:
\begin{align*}
  \Delta \ov \sC_{\lbi_{\Phi_{\mathrm{exp}}}, \sH}(h, x, x')
  & \leq \paren*{\eta'  e^{-h'} +  (1 - \eta') e^{h'}} \paren*{\sqrt{\eta  e^{-h}} - \sqrt{(1 - \eta) e^{h}}}^2\\
  & \qquad + \paren*{\eta  e^{-h} +  (1 - \eta) e^{h}} \paren*{\sqrt{\eta'  e^{-h'}} - \sqrt{(1 - \eta') e^{h'}}}^2.
\end{align*}
Therefore, we have
\begin{equation*}
 \Delta \ov \sC_{\lbi_{\Phi_{\mathrm{exp}}}, \sH}(h, x, x') 
\leq \sC_{\ell_{\Phi_{\mathrm{exp}}}}(h, x') \,
\Delta \sC_{\ell_{\Phi_{\mathrm{exp}}},\sH}(h, x) + \sC_{\ell_{\Phi_{\mathrm{exp}}}}(h, x) \,
\Delta \sC_{\ell_{\Phi_{\mathrm{exp}}},\sH}(h, x').
\end{equation*}
By Theorem~\ref{Thm:new-bound-concave-ranking}, we obtain
\begin{align*}
  \sE_{\lbi_{\Phi_{\mathrm{exp}}}}(h) - \sE_{\lbi_{\Phi_{\mathrm{exp}}}}^*(\sH) + \sM_{\lbi_{\Phi_{\mathrm{exp}}}}(\sH) \leq 2 \sE_{\ell_{\Phi_{\mathrm{exp}}}}(h) \,
\paren*{\sE_{\ell_{\Phi_{\mathrm{exp}}}}(h) - \sE^*_{\ell_{\Phi_{\mathrm{exp}}}}(\sH) + \sM_{\ell_{\Phi_{\mathrm{exp}}}}(\sH)}.
\end{align*}
\end{proof}

\subsection{Proof of Theorem~\ref{Thm:new-bound-concave-ranking-general}}
\label{app:new-bound-concave-ranking-general}

\NewBoundConcaveRankingGeneral*
\begin{proof}
By the definition, we have
\begin{align*}
\sC_{\Phi}(h, x) &= \eta(x) \Phi(h(x)) + \paren*{1 - \eta(x)} \Phi(-h(x)) \\
\sC_{\Phi}(h, x') &= \eta(x') \Phi(h(x')) + \paren*{1 - \eta(x')} \Phi(-h(x'))\\
\ov \sC_{\lbi_{\Phi}}(h, x, x') &= \eta(x)\paren*{1 - \eta(x')} \Phi(h(x) - h(x'))
 + \eta(x')\paren*{1 - \eta(x)}\Phi(-h(x) + h(x')).
\end{align*}
Therefore, by taking the derivative, we have
\begin{align*}
\Delta \sC_{\Phi,\sH}(h, x) &= 0 \implies \eta(x) \Phi'(h(x))  = \paren*{1 - \eta(x)} \Phi'(-h(x)) \implies h(x) = \frac{1}{\nu} \log\paren*{\frac{\eta(x)}{1 - \eta(x)}}\\
\Delta \sC_{\Phi,\sH}(h, x') &= 0 \implies \eta(x') \Phi'(h(x'))  = \paren*{1 - \eta(x')} \Phi'(-h(x')) \implies h(x') = \frac{1}{\nu} \log\paren*{\frac{\eta(x')}{1 - \eta(x')}}.
\end{align*}
Therefore, $h(x) - h(x') = \frac{1}{\nu} \log\paren*{\frac{\eta(x) (1 - \eta(x'))}{\eta(x') (1 - \eta(x)) } } $. This satisfies that
\begin{equation*}
\eta(x)\paren*{1 - \eta(x')} \Phi'(h(x) - h(x'))
 = \eta(x')\paren*{1 - \eta(x)} \Phi'(-h(x) + h(x')),
\end{equation*}
which implies that $\Delta \ov \sC_{\lbi_{\Phi}, \sH}(h, x, x') = 0$ by taking the derivative.
\end{proof}

\subsection{Proof of Theorem~\ref{Thm:new-bound-concave-ranking-log}}
\label{app:new-bound-concave-ranking-log}

\NewBoundConcaveRankingLog*
\begin{proof}
\ignore{
\begin{align*}
& \sC_{\lbi_{\Phi_{\mathrm{log}}}}(h, x, x')\\
&  = \eta(x)(1 - \eta(x')) \Phi_{\mathrm{log}}(h(x) - h(x')) + \eta(x')(1 - \eta(x))\Phi_{\mathrm{log}}(h(x') - h(x))\\
&  = \eta(x)(1 - \eta(x'))\log_2\paren*{1 + e^{-h(x) + h(x')}} + \eta(x')(1 - \eta(x))\log_2\paren*{1 + e^{h(x) - h(x')}}\\
& \sC^*_{\lbi_{\Phi_{\mathrm{log}}},\sH_{\mathrm{all}}}(x, x')\\
&  =  \inf_{h \in\sH_{\mathrm{all}}}\sC_{\lbi_{\Phi_{\mathrm{log}}}}(h, x, x')\\
&  =  -\eta(x)(1 - \eta(x'))\log_2(\eta(x)(1 - \eta(x'))) - \eta(x')(1 - \eta(x))\log_2(\eta(x')(1 - \eta(x))).
\end{align*}

\begin{align*}
& \sC_{\Phi_{\mathrm{log}}}(h, x)
 = \eta(x) \Phi_{\mathrm{log}}(h(x)) + (1 - \eta(x))\Phi_{\mathrm{log}}(-h(x))\\
& = \eta(x) \log_2\paren*{1 + e^{-h(x)}} + (1 - t) \log_2\paren*{1 + e^{h(x)}}.\\
& \inf_{h\in\sH_{\mathrm{all}}} \sC_{\Phi_{\mathrm{log}}}(h, x)\\
&= -\eta(x) \log_2(\eta(x)) - (1 - \eta(x)) \log_2(1 - \eta(x))   
\end{align*}
}
To simplify the notation, we will drop the dependency on $x$. Specifically, we use $\eta$ to denote $\eta(x)$, $\eta'$ to denote $\eta(x')$, $h$ to denote $h(x)$, and $h'$ to denote $h(x')$. Thus, we can write:
\begin{align*}
   \Delta \ov \sC_{\lbi_{\Phi_{\mathrm{log}}}, \sH}(h, x, x')
  & \leq \eta (1 - \eta') \log \bracket*{1 + e^{-h + h'}}
  +  \eta' (1 - \eta) \log \bracket*{1 + e^{-h' + h}}\\
  & + \eta (1 - \eta') \log \bracket*{\eta (1 - \eta')}
  + \eta' (1 - \eta) \log \bracket*{\eta' (1 - \eta)}\\
  \Delta \sC_{\ell_{\Phi_{\mathrm{log}}}, \sH}(h, x)
  & = \eta \log \bracket*{1 + e^{-h}}
  +  (1 - \eta) \log \bracket*{1 + e^{h}}\\
  & + \eta \log \bracket*{\eta}
  + (1 - \eta) \log \bracket*{(1 - \eta)}.
\end{align*}
Let $\Phi_{\mathrm{log}}$ denote the logistic loss. $\Phi_{\mathrm{log}}$ is a convex
function and for any $x$, we have $\Phi_{\mathrm{log}}(2x) \leq 2 \Phi_{\mathrm{log}}(x)$.
To prove this last inequality, observe that for any $x \in \Rset$,
we have
\begin{align*}
  \Phi_{\mathrm{log}}(2x)
  = \log (1 + e^{-2x})
  \leq \log (1 + 2e^{-x} + e^{-2x})
  = \log ((1 + e^{-x})^2)
  & = 2 \log ((1 + e^{-x})) = 2 \Phi_{\mathrm{log}}(x).
\end{align*}
Thus, we can write $\Phi_{\mathrm{log}}(h - h') = \Phi_{\mathrm{log}}(\frac{2h}{2} - \frac{2h'}{2})
\leq \frac{1}{2} (\Phi_{\mathrm{log}}(2h) + \Phi_{\mathrm{log}}(-2h')) \leq \Phi_{\mathrm{log}}(h) + \Phi_{\mathrm{log}}(-h')$.
In light of this inequality, we can write
\begin{align*}
  \Delta \ov \sC_{\lbi_{\Phi_{\mathrm{log}}}, \sH}(h, x, x')
  & \leq \eta (1 - \eta') (\Phi_{\mathrm{log}}(h) + \Phi_{\mathrm{log}}(- h'))
  +  \eta' (1 - \eta) (\Phi_{\mathrm{log}}(h') + \Phi_{\mathrm{log}}(- h))\\
  & \qquad + \eta (1 - \eta') \log \bracket*{\eta (1 - \eta')}
  + \eta' (1 - \eta) \log \bracket*{\eta' (1 - \eta)}\\
  & = \eta (1 - \eta') (\Phi_{\mathrm{log}}(h) + \Phi_{\mathrm{log}}(- h'))
  +  \eta' (1 - \eta) (\Phi_{\mathrm{log}}(h') + \Phi_{\mathrm{log}}(- h))\\
  & \qquad + \eta (1 - \eta') \bracket*{\log \eta + \log (1 - \eta')}
  + \eta' (1 - \eta) \bracket*{\log \eta' + \log (1 - \eta)}.
\end{align*}
Therefore, we have
\begin{equation*}
 \Delta \ov \sC_{\lbi_{\Phi_{\mathrm{log}}}, \sH}(h, x, x') \leq \max\curl*{\eta', 1 - \eta'} \Delta \sC_{\ell_{\Phi_{\mathrm{log}}},\sH}(h, x) + \max\curl*{\eta, 1 - \eta} \Delta \sC_{\ell_{\Phi_{\mathrm{log}}},\sH}(h, x') .
\end{equation*}
By Theorem~\ref{Thm:new-bound-concave-ranking}, we obtain
\begin{align*}
  & \sE_{\lbi_{\Phi_{\mathrm{log}}}}(h) - \sE_{\lbi_{\Phi_{\mathrm{log}}}}^*(\sH) + \sM_{\lbi_{\Phi_{\mathrm{log}}}}(\sH)\\
  & \qquad \leq 2\E[\max\curl*{\eta(x), (1 - \eta(x))}]  \,
\paren*{\sE_{\ell_{\Phi_{\mathrm{log}}}}(h) - \sE^*_{\ell_{\Phi_{\mathrm{log}}}}(\sH) + \sM_{\ell_{\Phi_{\mathrm{log}}}}(\sH)}.
\end{align*}
\end{proof}

\subsection{Proof of Theorem~\ref{Thm:new-bound-concave-ranking-hinge}}
\label{app:new-bound-concave-ranking-hinge}

\NewBoundConcaveRankingHinge*
\begin{proof}
Consider the distribution that supports on $\curl*{(x_0, x_0')}$. Let $1 \geq \eta(x_0) > \eta(x'_0) > \frac{1}{2}$, and $h_0 = 1 \in \sH$. Then, for any $h \in \sH$,
\begin{align*}
\sC_{\ell_{\Phi_{\mathrm{hinge}}}}(h, x_0) = \eta(x_0) \max \curl*{0, 1 - h(x_0)}
  +  (1 - \eta(x_0)) \max \curl*{0, 1 + h(x_0)} \geq 2 (1 - \eta(x_0))\\
\sC_{\ell_{\Phi_{\mathrm{hinge}}}}(h, x'_0) = \eta(x'_0) \max \curl*{0, 1 - h(x'_0)}
  +  (1 - \eta(x'_0)) \max \curl*{0, 1 + h(x'_0)} \geq 2 (1 - \eta(x'_0)),
\end{align*}
where both equality can be achieved by $h_0 = 1$. Furthermore,
\begin{align*}
\ov \sC_{\lbi_{\Phi_{\mathrm{hinge}}}}(h_0, x_0, x'_0) &= \eta(x_0) (1 - \eta(x'_0)) \max\curl*{0, 1 - h_0(x_0) + h_0(x'_0)}\\
& \qquad +  \eta(x'_0) (1 - \eta(x_0))  \max\curl*{1 -h_0(x'_0) + h_0(x_0)}\\
& = \eta(x_0) (1 - \eta(x'_0)) +  \eta(x'_0) (1 - \eta(x_0))\\
\Delta \ov \sC_{\lbi_{\Phi_{\mathrm{hinge}}}, \sH}(h_0, x_0, x'_0) &= \eta(x_0) (1 - \eta(x'_0)) +  \eta(x'_0) (1 - \eta(x_0))\\
&\qquad - 2\min\curl*{ \eta(x_0) (1 - \eta(x'_0)), \eta(x'_0) (1 - \eta(x_0))}\\
& = \eta(x_0) - \eta(x'_0).
\end{align*}
Therefore, $\Delta \sC_{\ell_{\Phi_{\mathrm{hinge}}}, \sH}(h_0, x_0) = \Delta \sC_{\ell_{\Phi_{\mathrm{hinge}}}, \sH}(h_0, x'_0) = 0$, but $ \Delta \ov \sC_{\lbi_{\Phi_{\mathrm{hinge}}}, \sH}(h_0, x_0, x'_0) \neq 0$, which implies that $\ell_{\Phi_{\rm{hinge}}}$ is not $\sH$-calibrated with respect to $\sfL_{\Phi_{\rm{hinge}}}$.

Suppose that for all
$h\in \sH$, the following holds:
\begin{equation*}
\Delta \ov \sC_{\lbi_{\Phi_{\mathrm{hinge}}}, \sH}(h, x_0, x'_0) \leq \Gamma_1 \paren*{\alpha_1(h, x'_0) \,
\Delta \sC_{\ell_{\Phi_{\mathrm{hinge}}}, \sH}(h, x_0)} + \Gamma_2 \paren*{\alpha_2(h, x_0) \,
\Delta \sC_{\ell_{\Phi_{\mathrm{hinge}}}, \sH}(h, x'_0)}.
\end{equation*}
Let $h = h_0$, then, for any $1 \geq \eta(x_0) > \eta(x'_0) > \frac{1}{2}$, the following inequality holds:
\begin{equation*}
\eta(x_0) - \eta(x'_0) \leq \Gamma_1(0) + \Gamma_2(0).
\end{equation*}
This implies that $\Gamma_1(0) + \Gamma_2(0) \geq \frac{1}{2}$.
\end{proof}

\section{Generalization bounds}
\label{app:generalization}

Here, we show that all our derived enhanced $\sH$-consistency bounds can be used to provide novel enhanced generalization bounds in their respective settings.

\subsection{Standard multi-class classification}

Let $S = \paren*{(x_1, y_1), \ldots, (x_m, y_m)}$ be a finite sample drawn from
$\sD^m$.  We denote by $\h h_S$ the minimizer of the
empirical loss within $\sH$ with respect to the constrained loss $\Phi^{\rm{cstnd}}$:
\[
\h h_S = \argmin_{h \in \sH} \h \sE_{\Phi^{\rm{cstnd}}, S}(h) = \argmin_{h \in \sH} \frac{1}{m}\sum_{i = 1}^m \Phi^{\rm{cstnd}}(h, x_i,y _i).\]

Next, by using enhanced $\sH$-consistency bounds for constrained losses $\Phi^{\rm{cstnd}}$ in Theorem~\ref{Thm:bound_lee_finer}, we derive novel generalization bounds for the multi-class zero-one loss by upper bounding the surrogate
estimation error $\sE_{\Phi^{\rm{cstnd}}}(\h h_S) - \sE_{\Phi^{\rm{cstnd}}}^*(\sH)$ with the
complexity (e.g. the Rademacher complexity) of the family of functions
associated with $\Phi^{\rm{cstnd}}$ and $\sH$: $\sH_{\Phi^{\rm{cstnd}}}=\curl*{(x, y) \mapsto
  \Phi^{\rm{cstnd}}(h, x, y) \colon h \in \sH}$.

Let $\Rad_m^{\Phi^{\rm{cstnd}}}(\sH)$ be the Rademacher complexity of
$\sH_{\Phi^{\rm{cstnd}}}$ and $B_{\Phi^{\rm{cstnd}}}$ an upper bound of the constrained loss
$\Phi^{\rm{cstnd}}$. The following generalization bound for the multi-class zero-one loss holds.

\begin{restatable}[\textbf{Enhanced generalization bound with constrained losses}]
  {theorem}{BoundLeeFinerG}
\label{Thm:bound_lee_finer_g}
Assume that $\sH$ is symmetric and complete. Then, the following generalization bound holds for $\h h_S$: for any
$\delta > 0$, with probability at least $1-\delta$ over the draw of an
i.i.d sample $S$ of size $m$:
\begin{align*}
 \sE_{\ell_{0-1}}(\h h_S) - \sE^*_{\ell_{0-1}}(\sH) + \sM_{\ell_{0-1}}(\sH) \leq \Gamma \paren*{4
    \Rad_m^{\Phi^{\rm{cstnd}}}(\sH) + 2 B_{\Phi^{\rm{cstnd}}} \sqrt{\tfrac{\log
        \frac{2}{\delta}}{2m}} + \sM_{\Phi^{\mathrm{cstnd}}}(\sH)},
\end{align*}
where $\Gamma(x) = \frac{ \sqrt{2}\, x^{\frac{1}{2}}}{\paren*{{e^{\Lambda(\h h_S)}}}^{\frac{1}{2}}}$ for $\Phi(u) = e^{-u}$, $\Gamma(x) = \frac{x}{1 + \Lambda(\h h_S)}$ for $\Phi(u) = \max \curl*{0, 1 - u}$, and $\Gamma(x) = \frac{ x^{\frac{1}{2}}}{1 + \Lambda(\h h_S)}$ for $\Phi(u) = (1 - u)^2 1_{u \leq 1}$. Additionally, $\Lambda(\h h_S) = \inf_{x \in \sX} \max_{y\in \sY} \h h_S(x,y)$.
\end{restatable}
\begin{proof}
  By using the standard Rademacher complexity bounds \citep{MohriRostamizadehTalwalkar2018}, for any $\delta>0$,
  with probability at least $1 - \delta$, the following holds for all $h \in \sH$:
\[
\abs*{\sE_{\Phi^{\mathrm{cstnd}}}(h) - \h\sE_{\Phi^{\mathrm{cstnd}}, S}(h)}
\leq 2 \Rad_m^{\Phi^{\mathrm{cstnd}}}(\sH) +
B_{\Phi^{\mathrm{cstnd}}} \sqrt{\tfrac{\log (2/\delta)}{2m}}.
\]
Fix $\e > 0$. By the definition of the infimum, there exists $h^* \in
\sH$ such that $\sE_{\Phi^{\mathrm{cstnd}}}(h^*) \leq
\sE_{\Phi^{\mathrm{cstnd}}}^*(\sH) + \e$. By definition of
$\h h_S$, we have
\begin{align*}
  & \sE_{\Phi^{\mathrm{cstnd}}}(\h h_S) - \sE_{\Phi^{\mathrm{cstnd}}}^*(\sH)\\
  & = \sE_{\Phi^{\mathrm{cstnd}}}(\h h_S) - \h\sE_{\Phi^{\mathrm{cstnd}}, S}(\h h_S) + \h\sE_{\Phi^{\mathrm{cstnd}}, S}(\h h_S) - \sE_{\Phi^{\mathrm{cstnd}}}^*(\sH)\\
  & \leq \sE_{\Phi^{\mathrm{cstnd}}}(\h h_S) - \h\sE_{\Phi^{\mathrm{cstnd}}, S}(\h h_S) + \h\sE_{\Phi^{\mathrm{cstnd}}, S}(h^*) - \sE_{\Phi^{\mathrm{cstnd}}}^*(\sH)\\
  & \leq \sE_{\Phi^{\mathrm{cstnd}}}(\h h_S) - \h\sE_{\Phi^{\mathrm{cstnd}}, S}(\h h_S) + \h\sE_{\Phi^{\mathrm{cstnd}}, S}(h^*) - \sE_{\Phi^{\mathrm{cstnd}}}^*(h^*) + \e\\
  & \leq
  2 \bracket*{2 \Rad_m^{\Phi^{\mathrm{cstnd}}}(\sH) +
B_{\Phi^{\mathrm{cstnd}}} \sqrt{\tfrac{\log (2/\delta)}{2m}}} + \e.
\end{align*}
Since the inequality holds for all $\e > 0$, it implies:
\[
\sE_{\Phi^{\mathrm{cstnd}}}(\h h_S) - \sE_{\Phi^{\mathrm{cstnd}}}^*(\sH)
\leq 
4 \Rad_m^{\Phi^{\mathrm{cstnd}}}(\sH) +
2 B_{\Phi^{\mathrm{cstnd}}} \sqrt{\tfrac{\log (2/\delta)}{2m}}.
\]
Plugging in this inequality in the bounds of Theorem~\ref{Thm:bound_lee_finer} completes the proof.
\end{proof}
To the best
of our knowledge, Theorem~\ref{Thm:bound_lee_finer_g} provides the first enhanced
finite-sample guarantees, expressed in terms of minimizability gaps, for the estimation error of the minimizer of constrained losses with respect to the multi-class zero-one loss, incorporating a quantity
$\Lambda(\h h_S)$ depending on $\h h_S$.
The
proof uses our enhanced $\sH$-consistency bounds for constrained losses (Theorem~\ref{Thm:bound_lee_finer}), as well as standard Rademacher complexity guarantees.


\newpage

\subsection{Classification under low-noise conditions}

Let $S = \paren*{(x_1, y_1), \ldots, (x_m, y_m)}$ be a finite sample drawn from
$\sD^m$.  We denote by $\h h_S$ the minimizer of the
empirical loss within $\sH$ with respect to a surrogate loss $\ell$:
$
\h h_S = \argmin_{h \in \sH} \h \sE_{\ell, S}(h) = \argmin_{h \in \sH} \frac{1}{m}\sum_{i = 1}^m \ell(h, x_i,y _i).$

Next, by using enhanced $\sH$-consistency bounds for surrogate losses $\ell$ in Theorems~\ref{Thm:tsybakov-binary} and \ref{Thm:tsybakov-multi}, we derive novel generalization bounds for the binary and multi-class zero-one loss under low-noise conditions, by upper bounding the surrogate
estimation error $\sE_{\ell}(\h h_S) - \sE_{\ell}^*(\sH)$ with the
complexity (e.g. the Rademacher complexity) of the family of functions
associated with $\ell$ and $\sH$: $\sH_{\ell}=\curl*{(x, y) \mapsto
  \ell(h, x, y) \colon h \in \sH}$.

Let $\Rad_m^{\ell}(\sH)$ be the Rademacher complexity of
$\sH_{\ell}$ and $B_{\ell}$ an upper bound of the surrogate loss
$\ell$. The following generalization bounds for the binary and multi-class zero-one loss hold.

\begin{restatable}[\textbf{Enhanced binary generalization bound under the Tsybakov noise
  assumption}]{theorem}{TsybakovBinaryG}
  \label{Thm:tsybakov-binary-g}
  Consider a binary classification setting where the Tsybakov noise
  assumption holds. Assume that the following holds for all $h \in \sH$ and $x \in \sX$:
  $\Delta\sC_{\ell^{\rm{bi}}_{0-1},\sH}(h, x) \leq \Gamma \paren*{\Delta \sC_{\ell,\sH}(h,
  x)}$, with $\Gamma(x) = x^{\frac{1}{s}}$, for some $s \geq 1$.
  Then, for any $h \in \sH$,
  \[
  \sE_{\ell^{\rm{bi}}_{0-1}}(\h h_S) - \sE^*_{\ell^{\rm{bi}}_{0-1}}(\sH)
  \leq c^{\frac{s - 1}{s - \alpha(s - 1)}}\bracket*{4 \Rad_m^{\ell}(\sH) +
2 B_{\ell} \sqrt{\tfrac{\log (2/\delta)}{2m}} + \sM_{\ell}(\sH)}^{\frac{1}{s - \alpha(s - 1)}}.
  \]
\end{restatable}
\begin{proof}
  By using the standard Rademacher complexity bounds \citep{MohriRostamizadehTalwalkar2018}, for any $\delta>0$,
  with probability at least $1 - \delta$, the following holds for all $h \in \sH$:
\[
\abs*{\sE_{\ell}(h) - \h\sE_{\ell, S}(h)}
\leq 2 \Rad_m^{\ell}(\sH) +
B_{\ell} \sqrt{\tfrac{\log (2/\delta)}{2m}}.
\]
Fix $\e > 0$. By the definition of the infimum, there exists $h^* \in
\sH$ such that $\sE_{\ell}(h^*) \leq
\sE_{\ell}^*(\sH) + \e$. By definition of
$\h h_S$, we have
\begin{align*}
  & \sE_{\ell}(\h h_S) - \sE_{\ell}^*(\sH)\\
  & = \sE_{\ell}(\h h_S) - \h\sE_{\ell, S}(\h h_S) + \h\sE_{\ell, S}(\h h_S) - \sE_{\ell}^*(\sH)\\
  & \leq \sE_{\ell}(\h h_S) - \h\sE_{\ell, S}(\h h_S) + \h\sE_{\ell, S}(h^*) - \sE_{\ell}^*(\sH)\\
  & \leq \sE_{\ell}(\h h_S) - \h\sE_{\ell, S}(\h h_S) + \h\sE_{\ell, S}(h^*) - \sE_{\ell}^*(h^*) + \e\\
  & \leq
  2 \bracket*{2 \Rad_m^{\ell}(\sH) +
B_{\ell} \sqrt{\tfrac{\log (2/\delta)}{2m}}} + \e.
\end{align*}
Since the inequality holds for all $\e > 0$, it implies:
\[
\sE_{\ell}(\h h_S) - \sE_{\ell}^*(\sH)
\leq 
4 \Rad_m^{\ell}(\sH) +
2 B_{\ell} \sqrt{\tfrac{\log (2/\delta)}{2m}}.
\]
Plugging in this inequality in the bounds of Theorem~\ref{Thm:tsybakov-binary} completes the proof.
\end{proof}

\begin{restatable}[\textbf{Enhanced multi-class generalization bound under the Tsybakov noise
  assumption}]{theorem}{TsybakovMultiG}
  \label{Thm:tsybakov-multi-g}
  Consider a multi-class classification setting where the Tsybakov noise
  assumption holds. Assume that the following holds for all $h \in \sH$ and $x \in \sX$:
  $\Delta\sC_{\ell_{0-1},\sH}(h, x) \leq \Gamma \paren*{\Delta \sC_{\ell,\sH}(h,
  x)}$, with $\Gamma(x) = x^{\frac{1}{s}}$, for some $s \geq 1$.
  Then, for any $h \in \sH$,
  \[
  \sE_{\ell_{0-1}}(\h h_S) - \sE^*_{\ell_{0-1}}(\sH)
  \leq c^{\frac{s - 1}{s - \alpha(s - 1)}}\bracket*{4 \Rad_m^{\ell}(\sH) +
2 B_{\ell} \sqrt{\tfrac{\log (2/\delta)}{2m}} + \sM_{\ell}(\sH)}^{\frac{1}{s - \alpha(s - 1)}}.
  \]
\end{restatable}
\begin{proof}
  By using the standard Rademacher complexity bounds \citep{MohriRostamizadehTalwalkar2018}, for any $\delta>0$,
  with probability at least $1 - \delta$, the following holds for all $h \in \sH$:
\[
\abs*{\sE_{\ell}(h) - \h\sE_{\ell, S}(h)}
\leq 2 \Rad_m^{\ell}(\sH) +
B_{\ell} \sqrt{\tfrac{\log (2/\delta)}{2m}}.
\]
Fix $\e > 0$. By the definition of the infimum, there exists $h^* \in
\sH$ such that $\sE_{\ell}(h^*) \leq
\sE_{\ell}^*(\sH) + \e$. By definition of
$\h h_S$, we have
\begin{align*}
  & \sE_{\ell}(\h h_S) - \sE_{\ell}^*(\sH)\\
  & = \sE_{\ell}(\h h_S) - \h\sE_{\ell, S}(\h h_S) + \h\sE_{\ell, S}(\h h_S) - \sE_{\ell}^*(\sH)\\
  & \leq \sE_{\ell}(\h h_S) - \h\sE_{\ell, S}(\h h_S) + \h\sE_{\ell, S}(h^*) - \sE_{\ell}^*(\sH)\\
  & \leq \sE_{\ell}(\h h_S) - \h\sE_{\ell, S}(\h h_S) + \h\sE_{\ell, S}(h^*) - \sE_{\ell}^*(h^*) + \e\\
  & \leq
  2 \bracket*{2 \Rad_m^{\ell}(\sH) +
B_{\ell} \sqrt{\tfrac{\log (2/\delta)}{2m}}} + \e.
\end{align*}
Since the inequality holds for all $\e > 0$, it implies:
\[
\sE_{\ell}(\h h_S) - \sE_{\ell}^*(\sH)
\leq 
4 \Rad_m^{\ell}(\sH) +
2 B_{\ell} \sqrt{\tfrac{\log (2/\delta)}{2m}}.
\]
Plugging in this inequality in the bounds of Theorem~\ref{Thm:tsybakov-multi} completes the proof.
\end{proof}

To the best
of our knowledge, Theorems~\ref{Thm:tsybakov-binary-g} and \ref{Thm:tsybakov-multi-g} provide the first enhanced finite-sample guarantees, expressed in terms of minimizability gaps, for the estimation error of the minimizer of surrogate losses with respect to the binary and multi-class zero-one loss, under the Tsybakov noise assumption.
The
proofs use our enhanced $\sH$-consistency bounds (Theorems~\ref{Thm:tsybakov-binary} and \ref{Thm:tsybakov-multi}), as well as standard Rademacher complexity guarantees.

Note that our enhanced $\sH$-consistency bounds can also be combined with standard bounds of $\sE_{\ell}(\h h_S) - \sE_{\ell}^*(\sH)$ in the case of Tsybakov noise, which can yield a fast rate. For example, our enhanced $\sH$-consistency bounds in binary classification (Theorems~\ref{Thm:tsybakov-binary}) can be combined with the fast estimation rates described in \citep[Section 4]{bartlett2006convexity}, which can lead to enhanced finite-sample guarantees in the case of Tsybakov noise. This is even true in the case of $\sH = \sH_{\mathrm{all}}$, with a more favorable factor of one instead of $2^{\frac{s}{s - \alpha (s - 1)}}$\ignore{, which is always greater
than one}.

\subsection{Bipartite ranking}

Let $S = \paren*{(x_1, y_1), (x'_1, y'_1), \ldots, (x_m, y_m), (x'_m, y'_m)}$ be a finite sample drawn from
$(\sD \times \sD)^{m}$.  We denote by $\h h_S$ the minimizer of the
empirical loss within $\sH$ with respect to a classification
surrogate loss $\ell_{\Phi}$:
$
\h h_S = \argmin_{h \in \sH} \h \sE_{\ell_{\Phi}, S}(h) = \argmin_{h \in \sH} \frac{1}{m}\sum_{i = 1}^m \ell_{\Phi}(h, x_i,y _i).$

Next, by using enhanced $\sH$-consistency bounds for classification surrogate losses $\ell_{\Phi}$ in Theorems~\ref{Thm:new-bound-concave-ranking-exp} and \ref{Thm:new-bound-concave-ranking-log}, we derive novel generalization bounds for bipartite ranking
surrogate losses $\sfL_{\Phi}$, by upper bounding the surrogate
estimation error $\sE_{\ell_{\Phi}}(\h h_S) - \sE_{\ell_{\Phi}}^*(\sH)$ with the
complexity (e.g. the Rademacher complexity) of the family of functions
associated with $\ell_{\Phi}$ and $\sH$: $\sH_{\ell_{\Phi}}=\curl*{(x, y) \mapsto
  \ell_{\Phi}(h, x, y) \colon h \in \sH}$.

Let $\Rad_m^{\ell_{\Phi}}(\sH)$ be the Rademacher complexity of
$\sH_{\ell_{\Phi}}$ and $B_{\ell_{\Phi}}$ an upper bound of the classification surrogate
$\ell_{\Phi}$. The following generalization bounds for the bipartite ranking
surrogate losses $\sfL_{\Phi}$ hold.

\begin{restatable}[\textbf{Enhanced generalization bound with AdaBoost}]{theorem}{NewBoundConcaveRankingExpG}
\label{Thm:new-bound-concave-ranking-exp-g}
Assume that $\sH$ is complete. Then, for any hypothesis $h \in \sH$, we have
\begin{align*}
  & \sE_{\lbi_{\Phi_{\mathrm{exp}}}}(\h h_S) - \sE_{\lbi_{\Phi_{\mathrm{exp}}}}^*(\sH)
  + \sM_{\lbi_{\Phi_{\mathrm{exp}}}}(\sH)\\
  & \qquad \leq 2 \sE_{\ell_{\Phi_{\mathrm{exp}}}}(\h h_S) \,
\paren*{4 \Rad_m^{\ell_{\Phi_{\mathrm{exp}}}}(\sH) +
2 B_{\ell_{\Phi_{\mathrm{exp}}}} \sqrt{\tfrac{\log (2/\delta)}{2m}} + \sM_{\ell_{\Phi_{\mathrm{exp}}}}(\sH)}.
\end{align*}
\end{restatable}
\begin{proof}
  By using the standard Rademacher complexity bounds \citep{MohriRostamizadehTalwalkar2018}, for any $\delta>0$,
  with probability at least $1 - \delta$, the following holds for all $h \in \sH$:
\[
\abs*{\sE_{\ell_{\Phi_{\mathrm{exp}}}}(h) - \h\sE_{\ell_{\Phi_{\mathrm{exp}}}, S}(h)}
\leq 2 \Rad_m^{\ell_{\Phi_{\mathrm{exp}}}}(\sH) +
B_{\ell_{\Phi_{\mathrm{exp}}}} \sqrt{\tfrac{\log (2/\delta)}{2m}}.
\]
Fix $\e > 0$. By the definition of the infimum, there exists $h^* \in
\sH$ such that $\sE_{\ell_{\Phi_{\mathrm{exp}}}}(h^*) \leq
\sE_{\ell_{\Phi_{\mathrm{exp}}}}^*(\sH) + \e$. By definition of
$\h h_S$, we have
\begin{align*}
  & \sE_{\ell_{\Phi_{\mathrm{exp}}}}(\h h_S) - \sE_{\ell_{\Phi_{\mathrm{exp}}}}^*(\sH)\\
  & = \sE_{\ell_{\Phi_{\mathrm{exp}}}}(\h h_S) - \h\sE_{\ell_{\Phi_{\mathrm{exp}}}, S}(\h h_S) + \h\sE_{\ell_{\Phi_{\mathrm{exp}}}, S}(\h h_S) - \sE_{\ell_{\Phi_{\mathrm{exp}}}}^*(\sH)\\
  & \leq \sE_{\ell_{\Phi_{\mathrm{exp}}}}(\h h_S) - \h\sE_{\ell_{\Phi_{\mathrm{exp}}}, S}(\h h_S) + \h\sE_{\ell_{\Phi_{\mathrm{exp}}}, S}(h^*) - \sE_{\ell_{\Phi_{\mathrm{exp}}}}^*(\sH)\\
  & \leq \sE_{\ell_{\Phi_{\mathrm{exp}}}}(\h h_S) - \h\sE_{\ell_{\Phi_{\mathrm{exp}}}, S}(\h h_S) + \h\sE_{\ell_{\Phi_{\mathrm{exp}}}, S}(h^*) - \sE_{\ell_{\Phi_{\mathrm{exp}}}}^*(h^*) + \e\\
  & \leq
  2 \bracket*{2 \Rad_m^{\ell_{\Phi_{\mathrm{exp}}}}(\sH) +
B_{\ell_{\Phi_{\mathrm{exp}}}} \sqrt{\tfrac{\log (2/\delta)}{2m}}} + \e.
\end{align*}
Since the inequality holds for all $\e > 0$, it implies:
\[
\sE_{\ell_{\Phi_{\mathrm{exp}}}}(\h h_S) - \sE_{\ell_{\Phi_{\mathrm{exp}}}}^*(\sH)
\leq 
4 \Rad_m^{\ell_{\Phi_{\mathrm{exp}}}}(\sH) +
2 B_{\ell_{\Phi_{\mathrm{exp}}}} \sqrt{\tfrac{\log (2/\delta)}{2m}}.
\]
Plugging in this inequality in the bounds of Theorem~\ref{Thm:new-bound-concave-ranking-exp} completes the proof.
\end{proof}

\begin{restatable}[\textbf{Enhanced generalization bound with logistic regression}]{theorem}{NewBoundConcaveRankingLogG}
\label{Thm:new-bound-concave-ranking-log-g}
Assume that $\sH$ is complete. For any $x$, define $u(x) =
\max\curl*{\eta(x), 1 - \eta(x)}$. Then, for any hypothesis $h \in \sH$, we have
\begin{align*}
  & \sE_{\lbi_{\Phi_{\mathrm{log}}}}(\h h_S) - \sE_{\lbi_{\Phi_{\mathrm{log}}}}^*(\sH) + \sM_{\lbi_{\Phi_{\mathrm{log}}}}(\sH)\\
& \qquad \leq 2\E[u(X)]  \,
\paren*{4 \Rad_m^{\ell_{\Phi_{\mathrm{log}}}}(\sH) +
2 B_{\ell_{\Phi_{\mathrm{log}}}} \sqrt{\tfrac{\log (2/\delta)}{2m}} + \sM_{\ell_{\Phi_{\mathrm{log}}}}(\sH)}.
\end{align*}
\end{restatable}
\begin{proof}
  By using the standard Rademacher complexity bounds \citep{MohriRostamizadehTalwalkar2018}, for any $\delta>0$,
  with probability at least $1 - \delta$, the following holds for all $h \in \sH$:
\[
\abs*{\sE_{\ell_{\Phi_{\mathrm{log}}}}(h) - \h\sE_{\ell_{\Phi_{\mathrm{log}}}, S}(h)}
\leq 2 \Rad_m^{\ell_{\Phi_{\mathrm{log}}}}(\sH) +
B_{\ell_{\Phi_{\mathrm{log}}}} \sqrt{\tfrac{\log (2/\delta)}{2m}}.
\]
Fix $\e > 0$. By the definition of the infimum, there exists $h^* \in
\sH$ such that $\sE_{\ell_{\Phi_{\mathrm{log}}}}(h^*) \leq
\sE_{\ell_{\Phi_{\mathrm{log}}}}^*(\sH) + \e$. By definition of
$\h h_S$, we have
\begin{align*}
  & \sE_{\ell_{\Phi_{\mathrm{log}}}}(\h h_S) - \sE_{\ell_{\Phi_{\mathrm{log}}}}^*(\sH)\\
  & = \sE_{\ell_{\Phi_{\mathrm{log}}}}(\h h_S) - \h\sE_{\ell_{\Phi_{\mathrm{log}}}, S}(\h h_S) + \h\sE_{\ell_{\Phi_{\mathrm{log}}}, S}(\h h_S) - \sE_{\ell_{\Phi_{\mathrm{log}}}}^*(\sH)\\
  & \leq \sE_{\ell_{\Phi_{\mathrm{log}}}}(\h h_S) - \h\sE_{\ell_{\Phi_{\mathrm{log}}}, S}(\h h_S) + \h\sE_{\ell_{\Phi_{\mathrm{log}}}, S}(h^*) - \sE_{\ell_{\Phi_{\mathrm{log}}}}^*(\sH)\\
  & \leq \sE_{\ell_{\Phi_{\mathrm{log}}}}(\h h_S) - \h\sE_{\ell_{\Phi_{\mathrm{log}}}, S}(\h h_S) + \h\sE_{\ell_{\Phi_{\mathrm{log}}}, S}(h^*) - \sE_{\ell_{\Phi_{\mathrm{log}}}}^*(h^*) + \e\\
  & \leq
  2 \bracket*{2 \Rad_m^{\ell_{\Phi_{\mathrm{log}}}}(\sH) +
B_{\ell_{\Phi_{\mathrm{log}}}} \sqrt{\tfrac{\log (2/\delta)}{2m}}} + \e.
\end{align*}
Since the inequality holds for all $\e > 0$, it implies:
\[
\sE_{\ell_{\Phi_{\mathrm{log}}}}(\h h_S) - \sE_{\ell_{\Phi_{\mathrm{log}}}}^*(\sH)
\leq 
4 \Rad_m^{\ell_{\Phi_{\mathrm{log}}}}(\sH) +
2 B_{\ell_{\Phi_{\mathrm{log}}}} \sqrt{\tfrac{\log (2/\delta)}{2m}}.
\]
Plugging in this inequality in the bounds of Theorem~\ref{Thm:new-bound-concave-ranking-log} completes the proof.
\end{proof}

To the best
of our knowledge, Theorems~\ref{Thm:new-bound-concave-ranking-exp-g} and \ref{Thm:new-bound-concave-ranking-log-g} provide the first enhanced finite-sample guarantees, expressed in terms of minimizability gaps, for the estimation error of the minimizer of classification surrogate losses with respect to the bipartite ranking surrogate losses. Theorem~\ref{Thm:new-bound-concave-ranking-exp-g} is remarkable since it provide finite simple bounds of the estimation
error of the RankBoost loss function by that of AdaBoost. Significantly, Theorems~\ref{Thm:new-bound-concave-ranking-log-g} implies a parallel finding for logistic regression analogous to
that of AdaBoost.
The
proofs use our enhanced $\sH$-consistency bounds (Theorems~\ref{Thm:new-bound-concave-ranking-exp} and \ref{Thm:new-bound-concave-ranking-log}), as well as standard Rademacher complexity guarantees.

\ignore{
\section{Future work}
\label{app:future_work}

While we presented a general framework for establishing enhanced
$\sH$-consistency bounds that enables the derivation of more favorable
bounds in various scenarios—including standard multi-class
classification, binary and multi-class classification under Tsybakov
noise conditions, and bipartite ranking—extending our framework to
other learning scenarios, such as non-i.i.d.\ settings, would be an
interesting future research question. Moreover, although our work
is theoretical in nature, further empirical analysis, such as
verifying a favorable ranking property for logistic regression, is an
interesting study that we have initiated.
}

\end{document}